\documentclass{article}
\usepackage{kr}

\usepackage{times}
\usepackage{soul}
\usepackage{url}
\usepackage[hidelinks]{hyperref}
\usepackage[utf8]{inputenc}
\usepackage[small]{caption}
\usepackage{graphicx}
\usepackage{amsmath}
\usepackage{amsthm}
\usepackage{booktabs}
\usepackage{enumitem}

\usepackage{amsfonts} % Blackboard math symbols
\usepackage{xcolor} % Colors
\usepackage{cleveref} % Clever refs
\usepackage[noend]{algorithm2e} % Algorithms
\usepackage{tikz} % Graphs
\usetikzlibrary{positioning}
\usepackage{afterpage} % table positioning

\newtheorem{theorem}{Theorem}
\newtheorem*{theorem*}{Theorem}
\newtheorem{corollary}[theorem]{Corollary}
\newtheorem*{corollary*}{Corollary}
\newtheorem{proposition}[theorem]{Proposition}
\newtheorem*{proposition*}{Proposition}
\newtheorem{lemma}[theorem]{Lemma}
\newtheorem*{lemma*}{Lemma}
\newtheorem{definition}[theorem]{Definition}
\newtheorem*{definition*}{Definition}

\SetKwInOut{Input}{Input} % Input
\SetKwInOut{Output}{Output} % Output
\SetKwComment{Comment}{/* }{ */} % Comments
\DontPrintSemicolon % No semicolon at end of lines
\crefname{algocf}{alg.}{algs.} % Cref format for algorithms
\Crefname{algocf}{Algorithm}{Algorithms}
\RestyleAlgo{ruled} % Title at top with lines
\LinesNumbered % Use line numbers

\definecolor{bluecite}{HTML}{0875b7}
\hypersetup{
  colorlinks=true,
  citecolor=bluecite,
  linkcolor=magenta
}

\newcommand{\multisetopen}[0]{\{\!\!\{} % Open multiset brackets
\newcommand{\multisetclose}[0]{\}\!\!\}} % Close multiset brackets
\newcommand{\col}{\text{Col}} % Set of colours

\newcommand{\changemarker}[1]{#1} % accepted

\title{Logical Expressivity and Explanations for\\ Monotonic GNNs with Scoring Functions}

\author{%
Matthew Morris$^1$\and
David J. Tena Cucala$^2$\and
Bernardo Cuenca Grau$^1$
\affiliations
$^1$Department of Compute Science, University of Oxford\\
$^2$Department of Computer Science, Royal Holloway, University of London
\emails
\{matthew.morris, bernardo.cuenca.grau\}@cs.ox.ac.uk,
david.tenacucala@rhul.ac.uk
}

\begin{document}

\maketitle

\begin{abstract}
Graph neural networks (GNNs) are often used for the task of link prediction: predicting missing binary facts in knowledge graphs (KGs).
To address the lack of explainability of GNNs on KGs, recent works extract Datalog rules from GNNs with provable correspondence guarantees. The extracted rules can be used to explain the GNN's predictions; furthermore, they can help characterise the expressive power of various GNN models. However, these works address only a form of link prediction based on a restricted, low-expressivity graph encoding/decoding method.
In this paper, we consider a more general and popular approach for link prediction where a scoring function is used to decode the GNN output into fact predictions.
We show how GNNs and scoring functions can be adapted to be monotonic, use the monotonicity to extract sound rules for explaining predictions, and leverage existing results about the kind of rules that scoring functions can capture.
We also define procedures for obtaining equivalent Datalog programs for certain classes of monotonic GNNs with scoring functions.
Our experiments show that, on link prediction benchmarks, monotonic GNNs and scoring functions perform well in practice and yield many sound rules.
\end{abstract}
\section{Introduction}

% (1) Neural methods are used for KG completion.
Knowledge graphs (KGs) \cite{DBLP:journals/csur/HoganBCdMGKGNNN21} find use in a variety of applications \cite{yang2017leveraging,hamilton2017inductive,wang2018ripplenet} and are often represented as sets of unary and binary facts,
where vertices correspond to constants, edges to binary facts, and vertex labels to unary facts.
However, KGs are often incomplete -- the field of KG completion aims to solve this by predicting missing facts that hold in its (unknown a-priori) complete version.
A number of solutions have been proposed for KG completion, including embedding-based approaches with distance-based scoring functions \cite{abboud2020boxe}, tensor product scoring functions \cite{yang2015embedding}, 
recurrent neural networks \cite{sadeghian2019drum}, differentiable reasoning \cite{rocktaschel2017end,evans2018learning}, and language models \cite{xie2022discrimination}.

% (2) GNNs are a common solution used for this.
KG completion methods based on graph neural networks (GNNs)~\cite{ioannidis2019recurrent,pflueger2022gnnq,morris2024relational}, including R-GCN \cite{schlichtkrull2018modeling} and its extensions \cite{tian2020ra,cai2019transgcn,vashishth2019composition,yu2021knowledge,shang2019end,liu2021indigo}, are widely popular as GNNs can take advantage of the structural information in the graph.
As a result, GNN models have highly desirable properties such as invariance under graph isomorphisms and the ability to be applied to graphs of any size.

A limitation of GNN models is that applying them to a
KG produces real vector embeddings for the graph's vertices but leaves the edges untouched.
Therefore, to predict new edges (the task known as link prediction), many approaches use a decoder based on a \emph{scoring function}~\cite{wang2017knowledge}---a learnable mapping of binary predicates and pairs of embeddings to a real-valued score---which can be easily applied to the embeddings of each pair of vertices in the graph.
The score can then be compared against a threshold to decide whether a link should be predicted. This is a very popular approach for link prediction with GNNs \cite{zhang2022graph}, including the pioneering R-GCN \cite{schlichtkrull2018modeling}, which uses DistMult \cite{yang2015embedding} as a decoder.

% (4) Expressivity and explainability are important.
Despite their effectiveness, the predictions of models using GNNs and scoring functions cannot easily be explained, verified, and interpreted \cite{garnelo2019reconciling}.
In contrast, logic-based and neuro-symbolic approaches to KG completion often yield rules which can be used to explain the model's predictions.
Such methods include RuleN \cite{meilicke2018fine}, AnyBURL \cite{meilicke2018fine}, RNNLogic \cite{qu2020rnnlogic}, Neural-LP \cite{yang2017differentiable}, and DRUM \cite{sadeghian2019drum}.
To ensure that the rules truly capture the reasons why the model makes a particular
prediction, it is important to ensure that they are \emph{faithful} to the model,
in the sense that applying the rules to an arbitrary dataset produces the same result as the model.
% (5) Existing results for expressible and explainable neural methods.
Thus, there is growing interest in models whose predictions can be captured by faithful rules \cite{cucala2022faithful,wang2023faithful}.

For GNNs specifically, \citeauthor{cucala2021explainable} (\citeyear{cucala2021explainable}) consider GNNs with max aggregation and non-negative weights (along with other restrictions) and show that they can be faithfully characterized by programs of tree-like Datalog rules.
Similarly, monotonic max-sum GNNs \cite{cucala2023correspondence}, which encompass both max and sum aggregation, can be faithfully characterised by means of tree-like Datalog programs with inequalities.
However, these models do not use scoring functions; instead, they rely on an encoding-decoding strategy that can predict new edges only between constants already linked in the input graph.

% (6) Missing results for GNNs with scoring functions, despite this being the dominant approach.
On the other hand, \cite{huang2023theory} characterize the expressivity of various GNN and scoring function models using a relational Weisfeiler-Leman algorithm.
They also extend the results of \cite{barcelo2020logical} to GNNs that perform link prediction relative to a fixed query relation and source vertex, showing that they can only express logical classifiers from a specific fragment of first-order logic.
However, their work does not provide a means for obtaining faithful rules from the GNNs.
Thus, given the advantages and popularity of GNNs with scoring function decoders, there is an important open question as to how faithful rules may be obtained from them.

\paragraph{Our Contribution}
% Our contribution: expressivity and explainability for monotonic GNNs with monotonic scoring functions.
% In particular: (1) adaption of several popular scoring functions, (2) way to check for sound rules, (3) affect of scoring function expressivity on sound rules for GNN + scoring function, (4) finite space of rules for finding an equivalent program for a max GNN with monotonic increasing scoring function, and (5) finite space of rules for finding an equivalent program for a max-sum GNN with non-negative bilinear scoring function.
In this paper, we bound the logical expressivity of monotonic max-sum GNNs with monotonic scoring functions from above using Datalog rules with inequality, and we provide a method to extract faithful rules from these models to explain their predictions.
In particular, we show how many popular scoring functions can be made \emph{monotonically increasing}, provide a means for checking the soundness of Datalog rules that can be used to explain predictions of such models, and prove how existing results about which rule patterns are captured by different scoring functions affect the logical expressivity of the models.
We also define procedures for obtaining finite equivalent programs of \emph{tree-like} Datalog rules for (1) any monotonic GNN with max aggregation and monotonically increasing scoring function, and (2) any monotonic GNN with max-sum aggregation and a \emph{non-negative bilinear} scoring function.
We conduct experiments on the benchmark link prediction datasets of \cite{teru2020inductive} and also use the rule-based LogInfer evaluation framework described in \cite{liu2023revisiting} and \cite{morris2024relational}.
We find that performance either drops slightly or increases when restricting the model to be monotonic, and many sound rules can be extracted.
Therefore, our models maintain performance comparable to standard models while offering the added benefit of explainability.
\section{Background} \label{sec:background}
\paragraph{Datalog}

We fix a signature of $\delta \in \mathbb{N}$ unary predicates and a finite but arbitrary set of binary predicates (also referred to as relations).
We also consider countably infinite sets of variables and constants, disjoint with one another and the set of predicates.
A term is a variable or a constant. An atom is an expression of the form $R(t_1, ..., t_n)$, where each $t_i$ is a term and $R$ is a predicate with nonnegative integer arity $n$. 
A literal is an atom or any inequality $t_1 \not\approx t_2$. A literal is ground if it contains no variables. A fact is a ground atom, and a dataset $D$ is a finite set of facts.
We denote the set of constants of a dataset $D$ with $\texttt{con}(D)$.
A (Datalog) rule is a constant-free expression of the form
\begin{equation}
    B_1 \land ... \land B_n \rightarrow H, 
\label{eq:ruleform}
\end{equation} where  $B_1, ..., B_n$ are its body literals and $H$ is its head atom.
We do not require rules to be safe: there may be variables in the head that do not occur in a body atom.
Furthermore, to avoid redundant rules, we require that each inequality in the body of a rule mentions two different terms.
A (Datalog) program is a finite set of rules.
A substitution $\nu$ maps finitely many variables to constants.
For literal $\alpha$ and a substitution $\nu$ defined on each variable in $\alpha$, $\alpha \nu$ is obtained by replacing each occurrence of a variable $x$ in $\alpha$ with $\nu(x)$. For a dataset $D$ and a fact $B$, we write $D \models B$ if $B \in D$; furthermore, given constants $a_1$ and $a_2$, we write $D \models a_1 \not\approx a_2$ if $a_1 \neq a_2$, for uniformity.
We write $D \models A$ for a set or conjunction of facts $A$ if $D$ contains each fact in $A$.
The immediate consequence operator $T_r$ for a rule $r$ of form  
\eqref{eq:ruleform} maps a dataset $D$ to dataset $T_r(D)$ containing $H \nu$ for each substitution $\nu$ such that $D \models B_i \nu$ for each $i \in \{1, \dots, n\}$.
For a program $\mathcal{P}$, $T_{\mathcal{P}}(D) = \bigcup_{r \in \mathcal{P}} T_r(D)$.
Notice that $T_r$ and $T_{\mathcal{P}}$ are transformations from datasets to datasets.

\paragraph{Graphs}
We consider real-valued vectors and matrices. For $\mathbf{v}$ a vector and ${i > 0}$, $\mathbf{v}[i]$ denotes the $i$-th element of $\mathbf{v}$. For $\mathbf{A}$ a
matrix and ${i ,j> 0}$, $\mathbf{A}[i,j]$ denotes the element in row $i$ and column
$j$ of $\mathbf{A}$ (this notation generalises to higher-order tensors).
A function ${\sigma : \mathbb{R} \to \mathbb{R}}$ is \emph{monotonically increasing} if
${x \leq y}$ implies ${\sigma(x) \leq \sigma(y)}$.
We apply univariate functions to vectors element-wise.

We fix a finite set $\col$ of colours, with $|\col|$ equal to the number of binary predicates in the signature.
A graph $G$ is a tuple $\langle V, \{ E^c \}_{c \in \text{$\col$}}, \lambda \rangle$ where $V$ is a finite vertex set, each $E^c \subseteq V \times V$ is a set of directed edges, and $\lambda$ assigns to each $v \in V$ a vector of dimension $\delta$ (the number of unary predicates in the signature).
We assume that there is a one-to-one correspondence between \changemarker{potential} graph vertices and constants in the signature; we typically represent the vertex corresponding to constant $a$ by $v^a$.
When $\lambda$ is clear from the context, we abbreviate $\lambda(v)$ as $\mathbf{v}$ for each $v \in V$.
Graph $G$ is undirected if $E^c$ is symmetric for each $c \in$ $\col$ and is Boolean if $\mathbf{v}[i] \in \{0, 1\}$ for each $v \in V$ and  $i \in \{ 1, ..., \delta \}$.

\paragraph{Graph Neural Networks}
A graph neural network (GNN) $\mathcal{N}$
with ${L \geq 1}$ layers is a tuple
\begin{align}
\langle \; \{ \mathbf{A}_\ell \}_{\ell}, \; \{ \mathbf{B}^c_\ell\}_{c, \ell}, \; \{ \mathbf{b}_\ell \}_{\ell}, \; \{ \sigma_\ell \}_{\ell}, \; \{ \text{agg}_\ell \}_{\ell} \; \rangle,
\end{align}
for each ${\ell \in \{ 1, \dots, L \}}$ and ${c \in \col}$, matrices
$\mathbf{A}_\ell$ and $\mathbf{B}^c_\ell$ are of dimension
$\delta_\ell \times \delta_{\ell-1}$ with ${\delta_0 = \delta}$,
$\mathbf{b}_\ell$ is a vector of dimension $\delta_\ell$, $\sigma_\ell : \mathbb{R} \to \mathbb{R}$ is an activation function, and $\text{agg}_\ell$ is a function mapping real-valued multisets to real value\changemarker{s}.

Applying $\mathcal{N}$ to a graph induces a sequence of labels $\mathbf{v}_0, \mathbf{v}_1, ..., \mathbf{v}_L$ (sometimes denoted as  $\mathbf{v}_{\lambda_0}, \mathbf{v}_{\lambda_1}, ..., \mathbf{v}_{\lambda_L}$) for each vertex $v$ in the graph as follows. First,
$\mathbf{v}_0$ is the initial labelling of \changemarker{node $v$ in} the input graph; then, for each $1 \leq \ell \leq L$, $\mathbf{v}_\ell$ is defined by the following expression:
\begin{align*} %\label{align:def_gnn_update}
\sigma_\ell ( \mathbf{b}_\ell + \mathbf{A}_\ell \mathbf{v}_{\ell - 1} + \sum_{c \in \text{$\col$}} \mathbf{B}_\ell^c ~\text{agg}_\ell (\multisetopen \mathbf{u}_{\ell - 1} ~|~ (v,u) \in E^c \multisetclose) )
\end{align*}
The output of $\mathcal{N}$ is a graph with the same vertices and edges as the input graph, but where each vertex is labelled by $\mathbf{v}_L$. 
Note that this definition is the same as that of \cite{cucala2023correspondence}, except without the classification function.

In this paper, we consider GNNs with max-sum aggregation functions,
which generalise both max and sum aggregation.
For $k \in \mathbb{N}_0 \cup \infty$, a finite real multiset $S$, and $\ell := \text{min}(k, |S|)$, we define $\text{max-$k$-sum}(S) = \sum_{i=1}^\ell s_i$, where $s_1, ..., s_\ell$ are the $\ell$ largest numbers of $S$, including repeated values.
If $\ell = 0$, the sum is defined as $0$.
Note that max-$1$-sum is equivalent to max and max-$\infty$-sum is equivalent to sum.
If a GNN has $\text{agg}_\ell = \text{max-$k_\ell$-sum}$ for all $\ell \in \{ 1, ..., L \}$, where $k_\ell \in \mathbb{N}_0 \cup \infty$, we refer to it as a max-sum GNN (likewise for max GNNs, when every $k_\ell = 1$).

Finally, we say that a max-sum GNN is \emph{monotonic} if for all $\ell \in \{ 1, ..., L \}$ and $c \in \col$, all elements of $\mathbf{A}_\ell, \mathbf{B}_\ell^c$ are non-negative and $\sigma_\ell$ is monotonically increasing with non-negative range.

\paragraph{Scoring Functions}
A \emph{scoring function} $f$ with dimension $d_f \in \mathbb{N}$ and threshold $t_f \in \mathbb{R}$
defines $f(R, \mathbf{h}, \mathbf{t}) \mapsto \mathbb{R}$ for all vectors $\mathbf{h}, \mathbf{t} \in \mathbb{R}^{d_f}$ and relations $R$ in our fixed signature.
For a set of constants $C$ and function $e: C \to \mathbb{R}^{d_f}$ that assigns an embedding to each constant, we define $\mathcal{F}_f(C, e) = \{ R(a,b) ~|~ a,b \in C, R \text{ is a relation}, f(R, e(a), e(b)) \geq t_f \}$, i.e.\ the set of facts over constants of $C$ scored as true by $f$.
GNNs, as well as other methods, can be used to learn embedding functions $e$.

Throughout this paper, we will consider \emph{families} of scoring functions---sets of functions sharing a common form and parameterized by one or more values---where each family comprises all functions generated by varying the parameter values.
For example, RESCAL is the family of scoring functions of the form $f(R, \mathbf{h}, \mathbf{t}) = \mathbf{h}^\top \mathbf{M}_R \mathbf{t}$, where $\mathbf{M}_R$ is a real matrix of dimension $d_f \times d_f$.

\paragraph{Dataset Transformations Through GNNs and Scoring Functions}
A GNN $\mathcal{N}$ with $L$ layers and a scoring function $f$ with dimension $d_f = \delta_L$ can be used to realise a transformation
$T_{\mathcal{N}, f}$ from datasets to datasets --- going forward, when we mention a GNN and scoring function together, we assume $d_f = \delta_L$.
To perform the transformation, the input dataset must be 
first encoded into a graph that can be directly processed by the GNN.
To this end, we adopt the so-called \emph{canonical scheme} \cite{cucala2023correspondence}, which introduces a vertex for each constant in the dataset, and no other vertices;
then, colours in graphs correspond to binary predicates and indices of feature vectors labelling each vertex to unary predicates in the signature.
More precisely, $U_p$ denotes the predicate associated to index $p$, for each $p \in \{1, \dots, \delta\}$,
and $R^c$ denotes the predicate associated to colour $c$.
Then, the canonical encoding $\texttt{enc}(D)$ of a dataset $D$
is the Boolean graph with vertex $v^a$ for each constant $a$ in $D$, a $c$-coloured edge $(v^a,v^b)$ for each
fact $R^c(a,b) \in D$, and a vector $\mathbf{v}^a$ labelling each vertex $v^a$ such that vector component $\mathbf{v}^a[p]$ is set to $1$ if and only if $U_p(a) \in D$, for each $p \in \{1, \dots, \delta\}$.

Given a scoring function $f$, we define the application of the decoder $\texttt{dec}_f$ to a labelled graph $G$ as producing the dataset containing $R^c(a,b)$ for each relation $R^c$ and pair of constants $a,b$ corresponding to vertices in $G$ such that $f(R^c, \mathbf{v}^a, \mathbf{v}^b) \geq t_f$.
Thus, we have $\texttt{dec}_f(\mathcal{N}(\texttt{enc}(D)))$

\begin{align*}
&= \{ R(a, b) ~|~ a,b \in \texttt{con}(D), f(R, \mathbf{v^a_L}, \mathbf{v^b_L}) \geq t_f  \}
\end{align*}

The \emph{canonical} dataset transformation induced by a GNN $\mathcal{N}$ and a scoring function $f$ is then defined as: 
$T_{\mathcal{N}, f}(D) = \texttt{dec}_f(\mathcal{N}(\texttt{enc}(D))) .$
We often abbreviate $\mathcal{N}(\texttt{enc}(D))$ by $\mathcal{N}(D)$.
The GNN and scoring function combination can be trained end-to-end.

\paragraph{Soundness and Completeness}
A Datalog program or rule $\alpha$ is sound for a combination $(\mathcal{N}, f)$ of a GNN $\mathcal{N}$ and a scoring function $f$ if $T_\alpha(D) \subseteq T_{\mathcal{N}, f}(D)$ for each dataset $D$.
Conversely, $\alpha$ is complete for $(\mathcal{N}, f)$ if $T_{\mathcal{N}, f}(D) \subseteq T_\alpha(D)$ for each dataset $D$.
Finally, we say that $\alpha$ is equivalent to $(\mathcal{N}, f)$ if it is both sound and complete for $(\mathcal{N}, f)$.
\section{Sound Rules} \label{sec:sound_rules}
In this section, we propose adaptions to several existing families of scoring functions to make them amenable to rule extraction when used with a monotonic max-sum GNN---specifically, we make them \emph{monotonically increasing}.
Furthermore, we provide a means for checking whether a Datalog rule is sound for a monotonic max-sum GNN and monotonically increasing scoring function.
Finally, we show how existing results about which rule patterns are captured by various scoring function families affect the soundness of rules when they are used with a GNN.

\paragraph{Monotonically Increasing Scoring Functions}
\cite{cucala2023correspondence} show that the ability to extract Datalog rules from plain monotonic max-sum GNNs crucially depends on the fact that these models are \emph{monotonic under injective homomorphisms}, meaning that they are agnostic to the particular constants used, and if the model is applied to an arbitrary dataset, then the output values will never decrease as new facts are added to the input.
We identify a property of scoring functions which ensures that combining such a function with a monotonic max-sum GNN produces a transformation that is still monotonic in the same way.

\begin{definition}
A scoring function $f$ is \emph{monotonically increasing} if for all relations $R$ and vectors $\mathbf{h}, \mathbf{h}', \mathbf{t}, \mathbf{t}' \in \mathbb{R}^{d_f}$ such that $0 \leq \mathbf{h}[i] \leq \mathbf{h}'[i]$ and $0 \leq \mathbf{t}[i] \leq \mathbf{t}'[i]$ $\forall i \in \{ 1, ..., d_f \}$, we have that $f(R, \mathbf{h}, \mathbf{t}) \leq f(R, \mathbf{h}', \mathbf{t}')$.
\end{definition}

This property guarantees that a scoring function will never decrease its output for a pair of embeddings $\mathbf{h}$ and $\mathbf{t}$ if the values of those embeddings increase.
We now consider various scoring function families from the literature and show that some are never monotonically increasing, whereas  others can be restricted to be monotonically increasing.

There are a variety of ``distance-based models'' used for scoring binary facts.
For example, TransE \cite{bordes2013translating} has the scoring function $f(R, \mathbf{h}, \mathbf{t}) = - || \mathbf{h} + \mathbf{r}_R - \mathbf{t} ||_2$, where each $\mathbf{r}_R$ is a vector of size $d_f$ representing a fixed embedding for relation $R$.
Since this involves a sum in which $\mathbf{h}$ and $\mathbf{t}$ have different signs, there is no collection of vectors $\{ \mathbf{r}_R, \ldots \}$ such that the corresponding scoring function is monotonically increasing.
This problem occurs in all the distance-based models, including UM \cite{bordes2012joint}, TransH \cite{wang2014knowledge}, TransR \cite{lin2015learning}, TransD \cite{ji2015knowledge}, TransSparse \cite{ji2016knowledge}, TransM \cite{fan2014transition}, ManifoldE \cite{xiao2016one}, TransF \cite{feng2016knowledge}, TransA \cite{xiao2015transa}, RotatE \cite{sun2018rotate}, and HAKE \cite{zhang2020learning}.
A similar issue occurs with some non-distance-based models including ComplEx \cite{trouillon2016complex}, QuatE \cite{zhang2019quaternion}, and DualE \cite{cao2021dual}, when the computation is considered in terms of real vectors, and is explained in \Cref{app:non_monotonic_increasing_scoring_functions}.

In contrast, the following scoring function families rely on combinations of multiplication, addition, and non-linear functions; they can be modified by restricting their main parameters---usually expressed as vectors, matrices, tensors, or neural network weights---to contain only non-negative values, and by making all non-linear functions in them monotonically increasing and non-negative (e.g., ReLU).

For example, when restricting RESCAL~\cite{nickel2011three} such that each $\mathbf{M}_R$ contains only non-negative numbers, $f$ is monotonically increasing.
This approach yields monotonically increasing scoring functions for the following other families: DistMult \cite{yang2015embedding}, ANALOGY \cite{liu2017analogical}, HolE \cite{nickel2016holographic}, TATEC \cite{garcia2014effective}, TuckER \cite{balazevic2019tucker}, SimplE \cite{kazemi2018simple}, SME \cite{bordes2014semantic}, NTN \cite{socher2013reasoning}, SLM \cite{socher2013reasoning}, MLP \cite{dong2014knowledge}, NAM \cite{liu2016probabilistic}, and ConvE \cite{dettmers2018convolutional}.
These families are fully described in \Cref{app:monotonic_increasing_semantic_matching} and \Cref{app:monotonic_increasing_neural_network}.
This list is non-exhaustive -- there may be other scoring function families that can also be modified in this way or others to yield monotonically increasing scoring functions.

\paragraph{Checking Soundness of Rules}
We provide a means for checking whether a given Datalog rule is sound for a monotonic max-sum GNN and monotonically increasing scoring function.
First, we establish that using a monotonic max-sum GNN and monotonically increasing scoring function yields a monotonic dataset transformation.

\begin{lemma} \label{lemma:monotonic_extension}
Let $\mathcal{N}$ be a monotonic max-sum GNN and $f$ a monotonically increasing scoring function.
Then for all datasets $D, D'$ such that $D \subseteq D'$, $T_{\mathcal{N}, f}(D) \subseteq T_{\mathcal{N}, f}(D')$.
\end{lemma}

The lemma is proved in \Cref{app:proof:lemma:monotonic_extension} and relies on vertex labels yielded by the GNN not decreasing when facts are added to the input dataset, and outputs of $f$ not decreasing as its input vectors increase.
In addition to this, max-sum GNNs and the canonical encoding are agnostic to the particular constants used in the input dataset: they depend only on the structure.
These two results show that monotonic max-sum GNNs with monotonically increasing scoring functions are monotonic under injective homomorphisms.
%, as has equivalently been shown for standard max-sum GNNs \cite{cucala2023correspondence}.
We now provide the proposition that allows one to check for sound rules, the proof of which is given in \Cref{app:proof:prop:sound_rule_checking} and relies on the above results, as well the fact that any dataset satisfying the body of $r$ contains one of the following datasets $D_\mu$ (up to constant renaming).

\begin{proposition} \label{prop:sound_rule_checking}
Let $\mathcal{N}$ be a monotonic max-sum GNN and $f$ a monotonically increasing scoring function.
Let $r$ be a Datalog rule with a binary head atom $H$, a (possibly empty) set $A$ of body atoms, and a (possibly empty) set $I$ of body inequalities.
For each variable $x$ in $r$, let $a_x, b_x$ be distinct constants uniquely associated with $x$.
Then, $r$ is sound for $(\mathcal{N}, f)$ if and only \changemarker{if} $H \mu \in T_{\mathcal{N}, f}(D_\mu)$ for each substitution $\mu$ mapping the variables of $r$ to constants in the set $\{ a_x ~|~ \text{$x$ is a variable of $r$} \}$ such that $\mu(x) \not= \mu(y)$ for each inequality $x \not\approx y \in I$, and
$D_\mu = A\mu \cup \{ R(b_x, \mu(x)) ~|~ \text{$x$ is a variable occurring in $H$ but not in $A$} \}$ where $R$ is an arbitrary but fixed binary predicate.
\end{proposition}

The inclusion of each $R(b_x, \mu(x))$ in $D_\mu$ accounts for unsafe rules.
A rule can be checked for soundness by enumerating each above-defined substitution $\mu$ and computing $T_{\mathcal{N}, f}(D_\mu)$; then the rule is sound if and only if $H \mu \in T_{\mathcal{N}, f}(D_\mu)$ for each such substitution.
For example, to check the soundness of rule $r: P_1(x,y) \land U_1(x) \rightarrow P_2(x,y)$, we have four possible substitutions. The first two are:
\begin{enumerate}[leftmargin=0.8cm]
    \item $\mu_1 = \{ x \mapsto a_x, y \mapsto a_x \}$ \\
    $D_{\mu_1} = \{ U_1(a_x), P_1(a_x, a_x) \}$
    \item $\mu_2 = \{ x \mapsto a_x, y \mapsto a_y \}$ \\
    $D_{\mu_2} = \{ U_1(a_x), P_1(a_x, a_y)
    \}$
\end{enumerate}
% \begin{enumerate}[leftmargin=0.75cm]
%     \item $\mu_1 = \{ x \mapsto a_x, y \mapsto a_x \}$ \\
%     $D_{\mu_1} = \{ U_1(a_x), P_1(a_x, a_x), R(b_y, a_x) \}$
%     \item $\mu_2 = \{ x \mapsto a_x, y \mapsto a_y \}$ \\
%     $D_{\mu_2} = \{ U_1(a_x), P_1(a_x, a_x), R(b_y, a_y) \}$
%     \item $\mu_3 = \{ x \mapsto a_y, y \mapsto a_x \}$ \\
%     $D_{\mu_3} = \{ U_1(a_y), P_1(a_y, a_y), R(b_y, a_x) \}$
%     \item $\mu_4 = \{ x \mapsto a_y, y \mapsto a_y \}$\\
%     $D_{\mu_4} = \{ U_1(a_y), P_1(a_y, a_y), R(b_y, a_y) \}$
% \end{enumerate}

and the others are symmetric to them.
If $H \mu_i \in T_{\mathcal{N}, f}(D_{\mu_i})$ for each substitution, then the soundness of $r$ follows, since any dataset satisfying the body of $r$ contains one of the above datasets, up to constant renaming.

\paragraph{Leveraging Scoring Function Results}
In this section, we consider GNNs (not necessarily monotonic, or even max-sum) and exploit known results \cite{sun2018rotate,abboud2020boxe} about the ability of a scoring function family to model different types of rules (referred to as ``patterns''), to bound the logical expressivity of the GNN and scoring function.
\cite{pavlovic2022expressive} formally define the capturing of rule patterns.
Intuitively, a scoring function family captures a pattern \emph{exactly} if there exists a scoring function in the family such that, for all constant embeddings, if the body of the pattern is satisfied then the head of the pattern is satisfied.
To formalise this using our definition of scoring functions, we first define a rule pattern.

\begin{definition} \label{def:rule_pattern}
A \emph{rule pattern} is a safe Datalog rule containing only binary predicates from a set of meta predicates $\mathcal{M} = \{ M_1, M_2, ... \}$, disjoint from the signature.
A rule $r$ \emph{conforms} to a rule pattern $\rho$ if it can be obtained from $\rho$ by replacing each distinct meta binary predicate with a different binary predicate from the signature.
\end{definition}

We can now formalise the capturing of rules.
In addition to defining ``capturing exactly'', we define ``capturing universally'', which generalises cases such as DistMult being symmetric.

\begin{definition} \label{def:scoring_func_capture}
Let $f$ be a scoring function, $C$ a set of constants, and $e: C \to \mathbb{R}^{d_f}$ an embedding function.
Then the tuple $(f, C, e)$ \emph{captures} a rule $B \rightarrow H$ if for every substitution $\mu$ such that $\mathcal{F}_f(C,e) \models B\mu$, we have $\mathcal{F}_f(C,e) \models H\mu$.

A family of scoring functions $F$ \emph{universally captures} a rule pattern $\rho$ if for all $f \in F$, constants $C$, $e: C \to \mathbb{R}^{d_f}$, and any rule $r$ conforming to $\rho$, we have $(f, C, e)$ captures $r$.

A family of scoring functions $F$ \emph{exactly captures} a rule pattern $\rho$ if for any rule $r$ conforming to $\rho$, there exists $f \in F$ such that for all constants $C$ and $e: C \to \mathbb{R}^{d_f}$, $(f, C, e)$ captures $r$.
\end{definition}

We give a sample of some of the pattern capturing results from the literature.
Consider the following rule patterns:
\begin{enumerate}[leftmargin=0.8cm]
    \item Hierarchy: $M_1(x,y) \rightarrow M_2(x,y)$
    \item Symmetry: $M_1(x,y) \rightarrow M_1(y,x)$
    \item Intersection: $M_1(x,y) \land M_2(x,y) \rightarrow M_3(x,y)$
    \item Inversion: $M_1(x,y) \leftrightarrow M_2(y,x)$
    \item Composition: $M_1(x,y) \land M_2(y,z) \rightarrow M_3(x,z)$
\end{enumerate}

For example, DistMult universally captures symmetry and exactly captures hierarchy  \cite{abboud2020boxe}.
TransE exactly captures inversion, composition, and intersection.

We now analyse the effect that a scoring function tuple capturing a rule has when used with a GNN.
We find that when a rule is captured across all constants and embeddings, certain rules being sound for the GNN and scoring function implies that there are other rules (not necessarily logically entailed by them) which must also be sound.
This limits the class of equivalent programs for the model, and thus its logical expressivity.
This is formalised in the following proposition, which is proved in \Cref{app:proof:prop:rule_captured_effect}.

\begin{proposition} \label{prop:rule_captured_effect}
Let $f$ be a scoring function and $B \rightarrow H$ a safe rule such that for all constants $C$ and $e: C \to \mathbb{R}^{d_f}$, $(f, C, e)$ captures $B \rightarrow H$.
Let $\mathcal{N}$ be a GNN.
Assume that for each atom $B_i$ of $B$, there is an inequality-free rule $A_i \rightarrow B_i$ that is sound for $(\mathcal{N}, f)$.
Then the rule $\bigwedge_i A_i \rightarrow H$ is sound for $(\mathcal{N}, f)$.
\end{proposition}

As a direct consequence of this, notice that when using a scoring function family that universally captures a pattern, the result applies to any rule conforming to the pattern and scoring function in the family.

\begin{corollary} \label{cor:rule_pattern_conjunction}
Let $F$ be a scoring function family that universally captures a rule pattern $\rho$, $B \rightarrow H$ a rule conforming to $\rho$, $\mathcal{N}$ a GNN, and $f \in F$.
Assume that for each atom $B_i$ of $B$, there is an inequality-free rule $A_i \rightarrow B_i$ sound for $(\mathcal{N}, f)$.
Then the rule $\bigwedge_i A_i \rightarrow H$ is sound for $(\mathcal{N}, f)$.
\end{corollary}
For example, let $F$ be a scoring function family that universally captures symmetry. Then for any GNN $\mathcal{N}$ and scoring function $f \in F$, if a rule $A \rightarrow P_1(x,y)$ is sound for $(\mathcal{N}, f)$, so is $A \rightarrow P_1(y,x)$.

A similar result applies to patterns captured exactly, except that it shows that there \emph{exists} a scoring function from the family such that when certain rules are sound, others must also be sound, regardless of the parameters of the GNN.

\begin{corollary} \label{cor:rule_pattern_exactly}
Let $F$ be a scoring function family that exactly captures a rule pattern $\rho$, and $B \rightarrow H$ a rule conforming to $\rho$.
Then there exists $f \in F$ such that: for each GNN $\mathcal{N}$ and inequality-free rules $A_i \rightarrow B_i$ sound for $(\mathcal{N}, f)$ for   each $B_i$ in $B$, the rule $\bigwedge_i A_i \rightarrow H$ is sound for $(\mathcal{N}, f)$.
\end{corollary}

Note that the above $f \in F$ is the same that witnesses the capturing of the rule in the definition of capturing exactly. 
For example, let $F$ be a scoring function family that exactly captures intersection.
Then, there exists $f \in F$ such that for any GNN $\mathcal{N}$, if rules $A_1 \rightarrow P_1(x,y)$ and $A_2 \rightarrow P_2(x,y)$ are sound for $(\mathcal{N}, f)$, then so is $A_1 \land A_2 \rightarrow P_3(x,y)$.

As an application of the results in this section, \Cref{cor:rule_pattern_conjunction,cor:rule_pattern_exactly} allow one to use rules already known to be sound to derive new sound rules. This avoids both searching for the new rules and using \Cref{prop:sound_rule_checking} to verify their soundness.
\section{Equivalent Programs}
In this section, we provide a means to obtain finite equivalent Datalog programs from GNNs and scoring functions: we do this for arbitrary monotonic max GNNs with monotonically increasing scoring functions, and for monotonic max-sum GNNs with a particular subclass of monotonically increasing scoring functions.
This provides an upper bound on the expressive power of such models. 
To define the rule space, we first provide the definition of a $(p, o)$-tree-like formula $\varphi_x$ for variable $x$.

\begin{definition}
For each root variable $x$:
\begin{enumerate}[leftmargin=0.8cm]
    \item $\top$ is tree-like for $x$;
    \item for each unary predicate $U$, $U(x)$ is tree-like for $x$;
    \item for all formulas $\varphi_1, \varphi_2$ that share no variables and are tree-like for $x$, $\varphi_1 \land \varphi_2$ is tree-like for $x$;
    \item for each variable $x$, binary predicate $R$, and tree-like formulas $\varphi_1, ..., \varphi_n$ for distinct variables $y_1, ..., y_n$ where no $\varphi_i$ contains $x$ and no $\varphi_i$ and $\varphi_j$ with $i \not= j$ share a variable, the following is tree-like for $x$:
    
\begin{equation}
\bigwedge_{i=1}^n \Big(
R(x, y_i) \land \varphi_i
\Big)
\land \bigwedge_{1 \leq i \leq j \leq n}^n y_i \not\approx y_j
\end{equation}
\end{enumerate}
\smallskip

Let $\varphi$ be a tree-like formula and let $x$ be a variable in $\varphi$.
The \emph{fan-out} of $x$ in $\varphi$ is the number of distinct variables $y_i$ for which $R(x, y_i)$ is a conjunct of $\varphi$.
The \emph{depth} of $x$ is the maximal $n$ for which there exist variables $x_0, ..., x_n$ and predicates $R_1, ..., R_n$ such that $x_n = x$ and $R(x_{i-1}, x_i)$ is a conjunct of $\varphi$ for each $1 \leq i \leq n$.
The \emph{depth} of $\varphi$ is the maximum depth of a variable in $\varphi$.

For $p$ and $o$ natural numbers, a tree-like formula $\varphi$ is $(p, o)$-tree-like if, for each variable $x$ in $\varphi$, the depth $i$ of $x$ is at most $p$ and the fan-out of $x$ is at most $o(p - i)$.

\end{definition}

Using this, we define a \emph{tree-like} rule.

\begin{definition}
A Datalog rule is $(p, o)$\emph{-tree-like} if it is of the form $\varphi_x \land \varphi_y \rightarrow R(x, y)$, where $\varphi_x, \varphi_y$ share no variables, $\varphi_x$ is $(p, o)$-tree-like for $x$, and $\varphi_y$ is $(p, o)$-tree-like for $y$.
\end{definition}

\paragraph{Monotonic Max GNNs}
We now provide a theorem describing a Datalog program equivalent to a given monotonic max GNNs and monotonically increasing scoring function.

\begin{theorem} \label{thm:max_rule_shape}
Let $\mathcal{N}$ be a monotonic max GNN and $f$ a monotonically increasing scoring function.
Let $\mathcal{P}_\mathcal{N}$ be the Datalog program containing, up to variable renaming, each $(L, |\text{Col}| \cdot \delta_\mathcal{N})$-tree-like rule without inequalities that is sound for $(\mathcal{N}, f)$, where $\delta_\mathcal{N} = \text{max}(\delta_0, ..., \delta_L)$. Then $T_{\mathcal{N}, f}$ and $\mathcal{P}_\mathcal{N}$ are equivalent.
\end{theorem}

The full proof is given in \Cref{app:proof:thm:max_rule_shape}.
Although the form of the rules extracted is similar to that for plain monotonic max-sum GNNs, the proof of the theorem
differs substantially from that of \cite[Theorem 13]{cucala2023correspondence} due to (1) the inclusion of two variables $x, y$ (with corresponding vertices and constants) in the base case and (2) the use of the monotonicity of the scoring function.

\begin{proof}[Proof sketch]
We show that $T_{\mathcal{N}, f}(D) = T_{P_\mathcal{N}}(D)$ holds for every dataset $D$.
$T_{P_\mathcal{N}}(D) \subseteq T_{\mathcal{N}, f}(D)$ is trivial.
So let $\alpha$ be an arbitrary binary fact in $T_{\mathcal{N}, f}(D)$.
Observe that $\alpha$ can be of the form $R(s, t)$, for distinct constants $s, t$, or of the form $R(s, s)$, for a constant $s$.
In this sketch, we will consider only the $R(s, t)$ case -- the other is similar.
We construct a $(L, |\text{Col}| \cdot \delta_\mathcal{N})$-tree-like rule without inequalities $r$, such that $\alpha \in T_r(D)$ and $r$ is sound for $(\mathcal{N}, f)$.
Together, these will imply that $\alpha \in T_{P_\mathcal{N}}(D)$, as required for the proof.

Let $G = (V, E, \lambda)$ be the canonical encoding of $D$ and let $\lambda_0, ..., \lambda_L$ the vertex labelling functions arising from the application of $\mathcal{N}$ to $G$.
We inductively construct $\Gamma$, a conjunction of two formulas, one of which is tree-like for $x$ and the other for $y$.
During the inductive construction, we also define a substitution $\nu$ from the variables of $\Gamma$ to the set of constants in $D$, and a graph $U$ (without vertex labels) where each vertex in $U$ is of the form $u^x$ for $x$ a variable and each edge has a colour in Col.
We assign to each vertex in $U$ a \emph{level} between $0$ and $L$.

For the base case, we introduce fresh variables $x, y$, define $\nu(x) = s, \nu(y) = t$, introduce vertices $u^x$ and $u_y$, and extend $\Gamma$ with $U(x)$ for each $U(s) \in D$ (likewise for $y$ for $t$).
For the induction step, consider $1 \leq \ell \leq L$ and assume that all vertices of level greater than or equal to $\ell$ have been
already defined.
We then consider each vertex of the form $u^x$ of level $\ell$. Let $t = \nu(x)$.
For each colour $c \in \text{Col}$, each layer $1 \leq \ell' < \ell$, and each dimension $j \in \{ 1, ..., \delta_{\ell' - 1} \}$ for which there exists some $(v^t, w) \in E^c$, let $w$ be a maximum $c$-coloured neighbour of $v^t$ in feature vector index $j$ at layer $\ell'$ (when applying $\mathcal{N}$ to $G$).
Note that $w$ must be of the form $v^s$ for some constant $s$ in $D$.

We then introduce a fresh variable $y$ and define $\nu(y) = s$.
We introduce a vertex $u^y$ of level $\ell - 1$ and an edge $E^c(u^x, u^y)$ to $U$.
Finally, we append to $\Gamma$ the conjunction
$E^c(x, y) ~\land~ ~\bigwedge_{U(s) \in D} U(y)$,
where $U$ stands in for arbitrary unary predicates.
This completes the inductive construction.
We then let $H = R(x,y)$, and define our tree-like rule to be $\Gamma \rightarrow H$.
The construction of $\nu$ means $D \models \Gamma \nu$.
So we have $H \nu \in T_r(D)$, with $H \nu = \alpha$, so $\alpha \in T_r(D)$, as required.

We now show that $r$ is sound for $(\mathcal{N}, f)$.
Let $D'$ be an arbitrary dataset and $\alpha'$ an arbitrary ground atom such that $\alpha' \in T_r(D')$.
Then there exists a substitution $\nu'$ such that $D' \models \Gamma \nu'$.
Let the graph $G = (V', E', \lambda')$ be the canonical encoding of $D'$ and $\lambda_0', ..., \lambda_L'$ the functions labelling the vertices of $G'$ when $\mathcal{N}$ is applied to it.

The following statement is proved by induction, which relies on the monotonicity of $\mathcal{N}$: for each $0 \leq \ell \leq L$ and each vertex $u^x$ of $U$ whose level is at least $\ell$, we have $\mathbf{v}_{\lambda_\ell}[i] \leq \mathbf{p}_{\lambda_\ell'}[i]$ for each $i \in \{ 1, ..., \delta_\ell \}$, where $v = v^{\nu(x)}$ and $p = v^{\nu'(x)}$.

Note $\alpha'$ has the form $R(s', t')$ with $s' = \nu'(x)$ and $t' = \nu'(y)$.
Now let $v = v^s, w = v^t, p = v^{s'}, q = v^{t'}$.
Then $R(s, t) \in T_{\mathcal{N}, f}(D)$ implies $f(R, \mathbf{v}_{\lambda_L}, \mathbf{w}_{\lambda_L}) \geq t_f$.
The above property ensures that $\mathbf{v}_{\lambda_L}[i] \leq \mathbf{p}_{\lambda_L'}[i]$ and $\mathbf{w}_{\lambda_L}[i] \leq \mathbf{q}_{\lambda_L'}[i]$ for all $i \in \{ 1, ..., \delta_L \}$.
Thus, since $f$ is monotonically increasing, we have $f(R, \mathbf{v}_{\lambda_L}, \mathbf{w}_{\lambda_L}) \leq f(R, \mathbf{p}_{\lambda_L'}, \mathbf{q}_{\lambda_L'})$.
So $f(R, \mathbf{p}_{\lambda_L'}, \mathbf{q}_{\lambda_L'}) \geq t_f$, which implies that $R(s', t') \in T_{\mathcal{N}, f}(D')$.
\end{proof}

The program $\mathcal{P}_\mathcal{N}$ is computable in principle by
enumerating all rules in the finite class of tree-like rules mentioned in the theorem, and then filtering out rules that do not pass the soundness check in \Cref{prop:sound_rule_checking}. 
This brute-force procedure can be optimised in practice: for example, by excluding all rules that subsume a known sound rule, since they are redundant; or by exploiting results about the expressiveness of scoring functions to avoid unnecessary soundness checks, as discussed at the end of \Cref{sec:sound_rules}.

Note that as a consequence of this theorem, rules conforming to common rule patterns (e.g.\ symmetry, inversion, hierarchy, composition, intersection, triangle, diamond \cite{liu2023revisiting}, fork, and cup \cite{morris2024relational}) can only be sound if they are subsumed by a sound tree-like rule.
For example, if a composition rule of the form $R_1(x,z) \land R_2(z,y) \rightarrow R_3(x,y)$ is sound for a monotonic max GNN and a monotonic scoring function, there must exist another sound tree-like rule that subsumes it, such as $R_1(x,z_1) \land R_2(z_2,y) \rightarrow R_3(x,y)$.

\paragraph{Monotonic Max-Sum GNNs}
When seeking to obtain equivalent programs for the larger class of monotonic max-sum GNNs instead of max GNNs, we can no longer consider all monotonically increasing scoring functions.
This is because there are a potentially unbounded number of neighbours that contribute to the computation during the execution of the max-sum GNN, so if we tried to follow the same approach as for monotonic max GNNs, we would have to consider an infinite number of tree-like rules.
To address this, we prove that an argument of the same type as that of \cite{cucala2023correspondence} for max-sum GNNs can also be made when a scoring function is considered. We show that the number of neighbours considered in each max-sum aggregation step can be restricted to a finite number without changing the output of the GNN and scoring function on any dataset.
This requires, however, an additional restriction on the scoring function: namely, 
it must be \emph{non-negative bilinear}. 
This property ensures that there is a 
scalar value $\alpha$ such that once a vertex feature output by a GNN is greater than it, increasing any input to the scoring function will not result in a change of the output of $T_{\mathcal{N},f}$.

\begin{definition}
A \emph{bilinear scoring function} $f$ is one such that for all relations $R$ and vectors $\mathbf{h}, \mathbf{t} \in \mathbb{R}^{d_f}$, we have $f(R, \mathbf{h}, \mathbf{t}) = \mathbf{h}^\top \mathbf{M}_R \mathbf{t}$, where $\mathbf{M}_R \in \mathbb{R}^{d_f \times d_f}$ is a matrix conditioned on $R$.
We say that $f$ is \emph{non-negative} if each $\mathbf{M}_R$ contains only non-negative entries.
\end{definition}

Notice that RESCAL, DistMult, and ANALOGY are bilinear scoring functions.
Also, TuckER can be rewritten as one, as shown in \Cref{app:monotonic_increasing_semantic_matching}.
In addition to restricting the scoring function, we also add the restriction that the activations functions $\{\sigma_\ell\}_{1 \leq\ \ell \leq L}$ of $\mathcal{N}$ must be unbounded.

The essence of our approach is to compute, for each layer $\ell$ of the GNN, a non-negative integer $\mathcal{C}_\ell$, called a \emph{capacity}, such that replacing the number $k_{\ell}$ in the aggregation function by $\mathcal{C}_{\ell}$ does not change the output of the model on any dataset. 
We compute the capacities $\mathcal{C}_\ell$
using a more complex variant of 
the algorithm of \cite[Algorithm 1]{cucala2023correspondence}, where the role
of their GNN classification threshold is
replaced by $\alpha$, defined as follows. In the following, the set $\mathcal{X}_{\ell, i}$ consists of all real numbers that can occur in label vector index $i$ at layer $\ell$, when a max-sum GNN $\mathcal{N}$ is applied to a dataset;
these sets can be computed as shown in
\cite[Definition 7]{cucala2023correspondence}.

\begin{definition} \label{def:bilinear_capacity}
Consider a monotonic max-sum GNN $\mathcal{N}$ and a non-negative bilinear scoring function $f$.
For each binary relation $R$ in the signature, where $\mathbf{M}_R$ is its relation matrix in $f$, we define: $\alpha_R := 1$ if all elements of $\mathbf{M}_R$ are $0$ or $\bigcup_i \mathcal{X}_{L,i} = \{ 0 \}$.
Otherwise, we define $\alpha_R := $ the least natural number such that $\alpha_R \cdot w \cdot \epsilon \geq t_f$, where $w$ is the least non-zero element of $\mathbf{M}_R$ and $\epsilon$ is the least non-zero element of $\bigcup_i \mathcal{X}_{L,i}$.
Then, we define $\alpha := \text{max}\{ \alpha_R ~|~ \text{for each binary relation $R$ in the signature} \}$.
\end{definition}

The following theorem then establishes that replacing the aggregation number $k_{\ell}$ by the capacity $\mathcal{C}_{\ell}$ leaves the output of the GNN and scoring function unchanged on any dataset.

\begin{theorem} \label{thm:capacity_maintains_equality}
Let $\mathcal{N}$ be a monotonic max-sum GNN with unbounded activation functions and $f$ a non-negative bilinear scoring function.
Let $\mathcal{N}'$ be the GNN obtained from $\mathcal{N}$ by replacing $k_\ell$ with the capacity $\mathcal{C}_\ell$ for each $\ell \in \{1, ..., L\}$.
Then for each dataset $D$, it holds that $T_{\mathcal{N}, f}(D) = T_{\mathcal{N}', f}(D)$.
\end{theorem}

The proof is given in \Cref{app:proof:thm:capacity_maintains_equality} and depends on the structure of non-negative bilinear scoring functions.
Using this, we prove the main result of this section, which is analogous to \Cref{thm:max_rule_shape}.
The theorem shows that an equivalent program can be constructed for any monotonic max-sum GNN $\mathcal{N}$ and non-negative bilinear scoring function $f$ using rules taken from the finite space of all $(L, |\text{Col}| \cdot \delta_\mathcal{N} \cdot \mathcal{C}_{\mathcal{N}, f})$-tree-like rules, where $\mathcal{C}_{\mathcal{N}, f} = \text{max}\{\mathcal{C}_1, ..., \mathcal{C}_L\}$.
The proof is given in \Cref{app:proof:thm:max_sum_rule_shape}.

\begin{theorem} \label{thm:max_sum_rule_shape}
Let $\mathcal{N}$ be a monotonic max-sum GNN with unbounded activation functions and $f$ a non-negative bilinear scoring function.
Let $\mathcal{P}_\mathcal{N}$ be the Datalog program containing, up to variable renaming, each $(L, |\text{Col}| \cdot \delta_\mathcal{N} \cdot \mathcal{C}_{\mathcal{N}, f})$-tree-like rule that is sound for $(\mathcal{N}, f)$, where $\delta_\mathcal{N} = \text{max}(\delta_0, ..., \delta_L)$. Then $T_{\mathcal{N}, f}$ and $\mathcal{P}_\mathcal{N}$ are equivalent.
\end{theorem}

\section{Experiments}
% Scoring functions: RESCAL, DistMult, TuckER, NAM

% Models: Sum, Max

% Variants: Normal, non-negative

% = 16 total model classes

% Seeds: 5 (10 later, if we have time)

% = 80 (160) runs per dataset

% Datasets: v1 of each of the Teru datasets. LogInfer-WN-hier. LogInfer-WN-sym. LogInfer-WN-cup\_nmhier = 6

% = 480 runs

%
%
%
%
%
%
%
%
%
%
%
%
% Max GNNs
\begin{table*}
\centering
\resizebox{2.11\columnwidth}{!}{
\begin{tabular}{lll|rrrrrr|r|rr}
\toprule
Dataset & Decoder & Model & \%Acc & \%Prec & \%Rec & \%F1 & AUPRC & Loss & \%SO & \#1B & \#2B \\

\midrule
WN-hier
& DistMult
& Standard
& 89.87 & 89.1  & 90.86 & 89.97 & 0.9423 & 0.11 & -     & -     & -    \\
&& Monotonic
& 85.05 & 79.18 & 95.52 & 86.43 & 0.809  & 1.41 & 52    & 171   & 26532 \\
& RESCAL
& Standard
& 92.69 & \textbf{90.79} & 95.02 & 92.86 & \textbf{0.9597} & 0.09 & -     & -     & -    \\
&& Monotonic
& 88.59 & 85.3  & 93.26 & 89.1  & 0.863  & 1.23 & \textbf{60}    & 104   & 17212 \\
& NAM
& Standard
& 92.1  & 90.22 & 94.46 & 92.28 & 0.9481 & 0.09 & -     & -     & -    \\
&& Monotonic
& 67.87 & 68.33 & 65    & 64.77 & 0.6577 & 1.64 & 48    & 141   & 19244 \\
& TuckER
& Standard
& \textbf{93}    & 90.51 & \textbf{96.08} & \textbf{93.21} & 0.9542 & 0.09 & -     & -     & -    \\
&& Monotonic
& 87.5  & 82.91 & 94.88 & 88.39 & 0.8206 & 1.36 & 56    & 89    & 14498 \\

\midrule
WN-sym
& DistMult
& Standard
& \textbf{96.6}  & \textbf{95.79} & 97.5  & \textbf{96.63} & \textbf{0.9929} & 0.09 & -     & -     & -    \\
&& Monotonic
& 92.34 & 86.92 & \textbf{99.68} & 92.86 & 0.9561 & 1.32 & 88    & 81    & 12845 \\
& RESCAL
& Standard
& 96.1  & 94.71 & 97.66 & 96.16 & 0.9889 & 0.07 & -     & -     & -    \\
&& Monotonic
& 94.27 & 89.72 & 100   & 94.58 & 0.9355 & 1.12 & \textbf{92}    & 42    & 7537 \\
& NAM
& Standard
& 95.72 & 94.58 & 97.02 & 95.78 & 0.9897 & 0.06 & -     & -     & -    \\
&& Monotonic
& 72.16 & 76.71 & 67.3  & 68.82 & 0.6775 & 1.79 & 40    & 88    & 12457 \\
& TuckER
& Standard
& 96.24 & 95.03 & 97.58 & 96.29 & 0.9888 & 0.07 & -     & -     & -    \\
&& Monotonic
& 91.31 & 85.97 & 98.84 & 91.95 & 0.8892 & 1.27 & 80    & 81    & 12825 \\

\midrule
WN-cup\_nmhier
& DistMult
& Standard
& 74.98 & 78.44 & 68.88 & 73.35 & 0.8139 & 0.17 & -     & -     & -    \\
&& Monotonic
& 69.69 & 71.21 & 66.45 & 68.61 & 0.6953 & 1.77 & 40    & 143   & 21366 \\
& RESCAL
& Standard
& \textbf{77.98} & 77.38 & 79.1  & \textbf{78.22} & 0.8152 & 0.15 & -     & -     & -    \\
&& Monotonic
& 73.21 & 75.91 & 67.99 & 71.72 & 0.7058 & 1.61 & 40    & 31    & 5006 \\
& NAM
& Standard
& 76.72 & 74.56 & \textbf{81.34} & 77.72 & \textbf{0.833}  & 0.14 & -     & -     & -    \\
&& Monotonic
& 59.74 & 59.04 & 63.1  & 57.47 & 0.5697 & 2.41 & \textbf{45}    & 354   & 50014 \\
& TuckER
& Standard
& 77.46 & 77.61 & 77.17 & 77.38 & 0.8144 & 0.15 & -     & -     & -    \\
&& Monotonic
& 71.51 & \textbf{80.15} & 57.27 & 66.78 & 0.664  & 1.68 & 40    & 22    & 3772 \\

\midrule
fb237v1
& DistMult
& Standard
& 55.55 & 53.08 & 97.1  & 68.62 & 0.9611 & 0.02 & -     & -     & -    \\
&& Monotonic
& 81.85 & 79.48 & 86    & 82.59 & 0.6884 & 0.80 & -     & 35366 & -    \\
& RESCAL
& Standard
& 58.8  & 55.38 & 96.5  & 70.23 & \textbf{0.9625} & 0.01 & -     & -     & -    \\
&& Monotonic
& \textbf{91}    & \textbf{96.72} & 84.9  & \textbf{90.41} & 0.6114 & 0.70 & -     & 4797  & -    \\
& NAM
& Standard
& 60.5  & 56.44 & 96.8  & 71.17 & 0.9552 & 0.01 & -     & -     & -    \\
&& Monotonic
& 64    & 58.66 & 97    & 73.03 & 0.9557 & 0.07 & -     & 155348 & -    \\
& TuckER
& Standard
& 54.35 & 52.31 & \textbf{98.9}  & 68.42 & 0.9597 & 0.01 & -     & -     & -    \\
&& Monotonic
& 90.5  & 92.07 & 88.7  & 90.34 & 0.6305 & -    & -     & 8511  & -    \\

\midrule
WN18RRv1
& DistMult
& Standard
& 89.76 & 83.13 & \textbf{100}   & 90.75 & 0.9337 & 0.02 & -     & -     & -    \\
&& Monotonic
& 91.52 & 85.49 & \textbf{100}   & 92.18 & 0.7922 & 0.79 & -     & 97    & 11995 \\
& RESCAL
& Standard
& 91.52 & 85.61 & 99.88 & 92.19 & 0.9306 & 0.01 & -     & -     & -    \\
&& Monotonic
& \textbf{95.82} & \textbf{92.38} & 99.88 & \textbf{95.98} & 0.7878 & 0.77 & -     & 52    & 6488 \\
& NAM
& Standard
& 78.61 & 70.96 & \textbf{100}   & 82.75 & 0.9287 & 0.01 & -     & -     & -    \\
&& Monotonic
& 74.73 & 66.51 & 99.88 & 79.83 & 0.9251 & 0.10 & -     & 202   & 20874 \\
& TuckER
& Standard
& 91.76 & 85.94 & \textbf{100}   & 92.41 & \textbf{0.9402} & 0.01 & -     & -     & -    \\
&& Monotonic
& 92.73 & 88.3  & 98.91 & 93.23 & 0.7787 & 0.81 & -     & 65    & 8134 \\

\midrule
nellv1
& DistMult
& Standard
& 57.65 & 54.51 & 92.71 & 68.56 & 0.9271 & 0.09 & -     & -     & -    \\
&& Monotonic
& 65.53 & 59.24 & \textbf{100}   & 74.39 & 0.9081 & 1.35 & -     & 631   & 125094 \\
& RESCAL
& Standard
& 58.94 & 55.01 & 99.06 & 70.72 & 0.9076 & 0.04 & -     & -     & -    \\
&& Monotonic
& 70.59 & \textbf{72.23} & 66.59 & 69.05 & 0.5372 & 1.05 & -     & 144   & 30154 \\
& NAM
& Standard
& 65.53 & 61.05 & 86.82 & 71.15 & 0.9127 & 0.03 & -     & -     & -    \\
&& Monotonic
& 53.24 & 51.69 & \textbf{100}   & 68.15 & \textbf{0.929}  & 0.44 & -     & 970   & 176836 \\
& TuckER
& Standard
& 55.53 & 53.04 & 99.06 & 69.06 & 0.9103 & 0.04 & -     & -     & -    \\
&& Monotonic
& \textbf{76.82} & 68.57 & 99.29 & \textbf{81.1}  & 0.6893 & 1.21 & -     & 265   & 53418 \\

\bottomrule
\end{tabular}
} % end resize box
\caption{Results for max GNNs. Loss is from the final epoch on the training set. AUPRC is from the validation set. Other metrics are computed on the test set. \%SO is the percentage of LogInfer rules that are sound for the model. \#1B and \#2B are the number of sound rules with one and two body atoms respectively.}
\label{tab:results:max_gnns}
\end{table*}

We train monotonic max-sum GNNs with monotonically increasing scoring functions across several link prediction datasets, showing that sound rules can be extracted in practice and that the restriction to monotonicity does not significantly decrease performance.
For the model architecture, we fix a hidden dimension of $50$, $2$ layers, and ReLU activation functions.
The GNN definition given in \Cref{sec:background}, chosen to correspond to that of \cite{cucala2023correspondence} and for ease of presentation, describes aggregation in the reverse direction of the edges.
In our experiments, we use the standard approach and aggregate in the direction of the edges.

We use GNNs with max aggregation, and GNNs with sum aggregation, as well as experimenting with $4$ different varieties of scoring function families: RESCAL, DistMult, TuckER, and NAM.
We train each model for $8000$ epochs, except for ones that use TuckER, which we train for $4000$, given our computational constraints and how much slower the TuckER models are.
For all trained models,
we compute standard classification metrics, such as precision, recall,
accuracy, F1 score, and 
area under the precision-recall curve (AUPRC).
To select the scoring function threshold, we evaluate the model on the validation set across candidate thresholds and select the one which maximises accuracy.
The candidate thresholds are the set of all scores produced on the validation input set.

Each training epoch, for each positive target fact, $10$ negative facts are generated by randomly corrupting the binary predicate.
These facts are filtered to ensure they do not contain any false negatives (i.e.\ facts that appear in the training set).
We originally corrupted the constants in the positive fact, but found that predicate corruption leads to significantly better performance by the baseline models on standard benchmarks \cite{teru2020inductive}.
We train all our models using binary cross entropy with logits (BCE) loss and the Adam optimizer with a standard learning rate of $0.001$ and weight decay of $5e^{-4}$.

We train models without restrictions as baselines (denoted by ``Standard''), as well as restricting the models to having non-negative weights (denoted by ``Monotonic''), by clamping negative weights to $0$ after each optimizer step, as in the approach of \cite{cucala2021explainable,morris2024relational}.
When clamping weights, we multiply by $50$ the term in the BCE loss function corresponding to the positive examples: without this, we found there to be insufficient positive signal for the training, leading to consistently positive gradients and thus weights that only tended to $0$ as training progressed;
the value of $50$ was arrived at by hyperparameter tuning on standard benchmark datasets.
We run each experiment across $5$ different random seeds and present the aggregated metrics.
Experiments are run using PyTorch Geometric, with a CPU on a Linux server.

\paragraph{Datasets}
We use 3 standard benchmarks: WN18RRv1, fb237v1, and nellv1 \cite{teru2020inductive},
each of which provides datasets for training, validation, and testing, as well as negative examples and positive targets.
Importantly, these benchmarks are also inductive, meaning that the validation and testing sets contain constants not seen during training, so approaches where embeddings are learned for each constant do not work for them.

We also utilize LogInfer \cite{liu2023revisiting}, a framework which augments a dataset by considering Datalog rules conforming to a particular pattern and adding the consequences of the rules to the dataset (we call these \emph{injected} rules).
We use the datasets LogInfer-WN-hier (WN-hier) and LogInfer-WN-sym (WN-sym) \cite{liu2023revisiting}, which are enriched with the hierarchy and symmetry patterns, respectively.
We also use LogInfer-WN-cup\_nmhier \cite{morris2024relational}, which was created using a mixture of monotonic and non-monotonic rules: rules from the ``cup'' ($R(x,y) \land S(y,z) \land T(w,x) \rightarrow P(x,y)$) and ``non-monotonic hierarchy'' ($R(x,y) \land \neg S(y,z) \rightarrow T(x,y)$) patterns, in this case.
We use the dataset to test whether the monotonic models can recover the monotonic rules in the dataset, despite the presence of non-monotonic rules.
For the LogInfer datasets, each training epoch, $10\%$ of the input facts are randomly set aside and used as ground truth positive targets, whilst the rest of the facts are used as input to the model.

Given that the datasets have no unary predicates, we introduce a dummy unary predicate that holds for every constant in the dataset, yielding the same initial embedding for every node in the encoding.
More details on this and other approaches to initial node features in the absence of unary predicates can be found in \Cref{app:node_features_no_unary}.

\paragraph{Rule Extraction}
On all datasets, we iterate over each Datalog rule in the signature with up to two body atoms and a binary head predicate, and count the number of sound rules, using \Cref{prop:sound_rule_checking} to check soundness.
On the fb237v1 dataset, we only check rules with one body atom, since the large number of predicates means that searching the space of rules with two body atoms intractable.
For datasets created with LogInfer \cite{liu2023revisiting}, 
which are obtained by appending the consequences of a known set of Datalog rules to a pre-existing dataset, we also check if these rules are sound.

\paragraph{Results}
The experimental results for max GNNs are provided in \Cref{tab:results:max_gnns}.
For each dataset and metric, we highlight in bold the best value achieved by a model variant for that metric.
Our findings definitively show that restricting GNNs and scoring functions to be monotonic does not adversely affect model performance in most cases (with the exception of models that use NAM).
In fact, in some scenarios, the restrictions cause a significant increase in performance.
Similar conclusions can be drawn from our results for sum GNNs as for max GNNs; for completeness, full results for monotonic sum GNNs are given in \Cref{tab:results:sum_gnns}, of \Cref{app:full_results}.

Some general patterns emerge.
Firstly, loss is always lower for the standard model than its monotonic counterpart.
This is due to the restriction on model weights preventing gradient descent from optimising the parameters to the same extent.
On the other hand, the monotonicity may help to prevent overfitting on the training set.
Likewise, AUPRC is consistently higher on the standard model than the monotonic version.
Used as an indication for model performance on the validation set, this suggests that the standard models may be overfitting on the validation facts and struggling to generalise to the test set, in comparison to the monotonic models.
Finally, we see that models using NAM as a scoring function consistently struggled when restricted to being monotonic: other scoring functions are thus more suitable when monotonicity is desired.

We obtained a number of sound rules for every monotonic model (columns \#1B and \#2B in \Cref{tab:results:max_gnns}), showing the efficacy of our rule-checking methodology.
On the LogInfer-WN datasets, there are $605$ possible rules with one body atom and $90508$ with two.
On fb237v1, there are $162000$ possible rules with one body atom.
On WN18RRv1, there are $405$ possible rules with one body atom and $49572$ with two.
On nellv1, there are $980$ possible rules with one body atom and $186592$ with two.
Notice that, for example, nearly all possible rules are sound when using monotonic NAM on nellv1, which corresponds to the perfect model recall and poor precision.
We show some randomly sampled sound rules in \Cref{tab:results:sample_rules} of \Cref{app:full_results}.

On the LogInfer datasets, we find consistently small drops in performance ($\approx 5\%$ less accuracy) when monotonic models are used.
However, the restriction enables us to check sound rules for the model: we find many sound rules with one or two body atoms, and also that around half (or sometimes more) of the injected LogInfer rules are sound for the model.
% By comparison, \cite{morris2024relational} found that $100\%$ of the injected rules were sound on WN-hier and WN-sym, when using a dataset encoding to perform link prediction.
On WN-cup\_nmhier, the monotonic models showed their ability to still learn and recover some of the injected monotonic rules, even in the presence of some injected rules that are explicitly non-monotonic.

On fb237v1, we see significantly better performance by monotonic models over their standard counterparts.
The same is seen on nellv1.
On WN18RRv1, the monotonic models again outperform the standard ones, but by smaller margins.
These results are somewhat surprising, but encouraging, since one would expect monotonic models to have a greater advantage on the LogInfer datasets, where the underlying patterns in the data are explicitly monotonic.
Finally, we note that the monotonic model performance on fb237v1 and WN18RRv1 is significantly better than those of MGNNs \cite{cucala2021explainable}, highlighting the benefits of using a scoring function as a decoder instead of their dataset encoding-decoding scheme, since the recall of their models is upper-bounded by the number of positive test pairs of constants that also appear in the input dataset.
\section{Conclusion}
In this paper, we showed how scoring functions can be made monotonically increasing, how sound Datalog rules can be found for monotonic GNNs with monotonic scoring functions, and showed how existing results about scoring functions capturing rule patterns impact the expressivity of GNN models that use those scoring functions.
We also provided ways to obtain an equivalent program for any monotonic GNN with max aggregation and monotonically increasing scoring function, and any monotonic GNN with max-sum aggregation and a scoring function that is non-negative bilinear.
We showed through our experiments that, in practice, the performance of GNNs with scoring functions does not drop substantially when applying the restrictions that make them monotonic, and in some cases increase the model performance.
We also showed that, in practice, many sound rules can be recovered from these monotonic models, which can be used to explain their predictions.

A limitation of this work is that we only consider non-negative bilinear scoring functions when obtaining equivalent programs for max-sum GNNs: the approach we describe may also work for other classes of monotonically increasing scoring functions\changemarker{, such as SimplE}.
For future work, we aim to consider more advanced training paradigms for the monotonic models.

\clearpage
\bibliographystyle{kr}
\bibliography{ref}

\newpage
\changemarker{
For the purpose of Open Access, the authors have applied a CC BY public copyright licence to any Author Accepted Manuscript (AAM) version arising from this submission.
\paragraph{Acknowledgments}
Matthew Morris is funded by an EPSRC scholarship (CS2122\_EPSRC\_1339049).
This work was also supported by Samsung Research UK, the EPSRC projects UKFIRES (EP/S019111/1) and ConCur (EP/V050869/1).
The authors would like to acknowledge the use of the University of Oxford Advanced Research Computing (ARC) facility in carrying out this work \url{http://dx.doi.org/10.5281/zenodo.22558}.
}

\appendix
\clearpage
\section{Theory}
\subsection{Non Monotonically Increasing Scoring Functions} \label{app:non_monotonic_increasing_scoring_functions}
ComplEx is defined as follows:

$$ f(R, \mathbf{h}, \mathbf{t}) = \text{Re}( \mathbf{h}^\top \mathbf{M}_R \overline{\mathbf{t}} ) , $$

where each $\mathbf{M}_R$ is a diagonal $d_f \times d_f$ matrix, $\mathbf{h}, \mathbf{t}, \mathbf{M}_R$ contain complex numbers, $\text{Re}$ denotes the real part of a complex number, and $\overline{\mathbf{t}}$ denotes the complex conjugate.
% Now from https://pytorch-geometric.readthedocs.io/en/latest/_modules/torch_geometric/nn/kge/complex.html#ComplEx
Splitting these vectors / matrices into their real and imaginary components, we obtains $\mathbf{h} = \mathbf{h}^{\text{im}} \cdot i + \mathbf{h}^{\text{re}},~ \mathbf{M}_R = \mathbf{M}_R^{\text{im}} \cdot i + \mathbf{M}_R^{\text{re}},~ \mathbf{t} = \mathbf{t}^{\text{im}} \cdot i + \mathbf{t}^{\text{re}}$, and thus:

\begin{align*}
f(R, \mathbf{h}, \mathbf{t}) =& \mathbf{h}^{\text{re}} \mathbf{M}_R^{\text{re}} \mathbf{t}^{\text{re}} + \mathbf{h}^{\text{im}} \mathbf{M}_R^{\text{re}} \mathbf{t}^{\text{im}} \\
& + \mathbf{h}^{\text{re}} \mathbf{M}_R^{\text{im}} \mathbf{t}^{\text{im}} - \mathbf{h}^{\text{im}} \mathbf{M}_R^{\text{im}} \mathbf{t}^{\text{re}} .
\end{align*}

To ensure monotonicity, we have to enforce one of $\mathbf{h}^{\text{im}}, \mathbf{M}_R^{\text{im}}, \mathbf{t}^{\text{re}} = 0$, otherwise we will have the same situation as for the distance-based models, where we have vectors with different signs.
Any one of these being zero eliminates two of the terms of the above equation, so for now assume only $\mathbf{M}_R^{\text{im}} = 0$, i.e. that the relation vector has only real numbers.
Then we obtain:

$$ f(R, \mathbf{h}, \mathbf{t}) = \mathbf{h}^{\text{re}} \mathbf{M}_R^{\text{re}} \mathbf{t}^{\text{re}} + \mathbf{h}^{\text{im}} \mathbf{M}_R^{\text{re}} \mathbf{t}^{\text{im}} . $$

This is a severely limited version of ComplEx, so we do not use it.
A similar problem occurs for QuatE and DualE, which use quaternions instead of the complex numbers.

\subsection{Monotonically Increasing Semantic Matching Models} \label{app:monotonic_increasing_semantic_matching}
When restricting the main parameters of the following--- expressed as vectors, matrices, or tensors---to contain only non-negative values, we obtain monotonically increasing scoring functions.

RESCAL is defined by

$$ f(R, \mathbf{h}, \mathbf{t}) = \mathbf{h}^\top \mathbf{M}_R \mathbf{t} , $$

where $\mathbf{M}_R \in \mathbb{R}^{d_f \times d_f}$.
When every $\mathbf{M}_R$ is diagonal, we obtain DistMult.
When we enforce normality ($\textbf{M}_R \textbf{M}_R^\top = \textbf{M}_R^\top \textbf{M}_R$) and commutativity ($\textbf{M}_R \textbf{M}_{R'} = \textbf{M}_{R'} \textbf{M}_R$) for all relations $R, R'$, we obtain ANALOGY.
TATEC is defined by:

$$ f(R, \mathbf{h}, \mathbf{t}) = \mathbf{h}^\top \mathbf{M}_R \mathbf{t} + \mathbf{h}^\top \mathbf{r}_R + \mathbf{t}^\top \mathbf{r}_R + \mathbf{h}^\top \mathbf{D} \mathbf{t} , $$

where $\mathbf{r}_R \in \mathbb{R}^{d_f}, \mathbf{M}_R \in \mathbb{R}^{d_f \times d_f}, \mathbf{D} \in \mathbb{R}^{d_f \times d_f}$.
HolE uses circular correlation operation to first combine the head and tail representations, and is defined as follows:

\begin{align*}
(\mathbf{h} \star \mathbf{t})[i] &= \sum_{j=1}^{d_f} \mathbf{h}[j] \cdot \mathbf{t}[(j + i) \text{ mod } d_f] \\
f(R, \mathbf{h}, \mathbf{t}) &= \mathbf{r}_R^\top (\mathbf{h} \star \mathbf{t}) , \\
\end{align*}

where $\mathbf{h} \star \mathbf{t}$ produces a vector of dimension $d_f$ and $\mathbf{r}_R \in \mathbb{R}^{d_f}$.
% https://pykeen.readthedocs.io/en/latest/api/pykeen.nn.modules.TuckERInteraction.html - the description given here contains a nice unrolling of the definition.
TuckER is defined by:
\begin{align*}
f(R, \mathbf{h}, \mathbf{t}) &= \mathbf{W} \times_1 \mathbf{h} \times_2 \mathbf{r}_R \times_3 \mathbf{t} \\
&= \sum_{1 \leq i, k \leq d_f, 1 \leq j \leq d_r} \mathbf{W}[i,j,k] \mathbf{h}[i] \mathbf{r}_R[j] \mathbf{t}[k] ,
\end{align*}

where $\times_n$ denotes the $n$-mode tensor product, $d_r \in \mathbb{N}$, $\mathbf{W} \in \mathbb{R}^{d_f \times d_r \times d_f}$, and $\mathbf{r}_R \in \mathbb{R}^{d_r}$.
Note that TuckER can be rewritten as a bilinear model:
$f(R, \mathbf{h}, \mathbf{t}) = \mathbf{h} \mathbf{M}_R \mathbf{t} = \sum_{i=1}^{d_f} \sum_{k=1}^{d_f} \mathbf{h}[i] \mathbf{M}_R[i,k] \mathbf{t}[k]$.
To do so, we define $\mathbf{M}_R \in \mathbb{R}^{d_f \times d_f}$ by $\mathbf{M}_R[i ,k] = \sum_{j=1}^{d_r} \mathbf{W}[i,j,k] \mathbf{r}_R[j]$.
Then:

\begin{align*}
f(R, \mathbf{h}, \mathbf{t}) &= \sum_{1 \leq i, k \leq d_f, 1 \leq j \leq d_r} \mathbf{W}[i,j,k] \mathbf{h}[i] \mathbf{r}_R[j] \mathbf{t}[k] \\
&= \sum_{i=1}^{d_f} \sum_{k=1}^{d_f} \sum_{j=1}^{d_r} \mathbf{W}[i,j,k] \mathbf{h}[i] \mathbf{r}_R[j] \mathbf{t}[k] \\
&= \sum_{i=1}^{d_f} \sum_{k=1}^{d_f} \mathbf{h}[i] \mathbf{t}[k]  \sum_{j=1}^{d_r} \mathbf{W}[i,j,k] \mathbf{r}_R[j] \\
&= \sum_{i=1}^{d_f} \sum_{k=1}^{d_f} \mathbf{h}[i] \mathbf{t}[k] \mathbf{M}_R[i ,k] \\
&= \sum_{i=1}^{d_f} \sum_{k=1}^{d_f} \mathbf{h}[i] \mathbf{M}_R[i ,k] \mathbf{t}[k] \\
&= \mathbf{h} \mathbf{M}_R \mathbf{t}
\end{align*}

SimplE (Kazemi and Poole, 2018) is defined as follows:

$$ f(R, \mathbf{h}, \mathbf{t}) = \frac{1}{2}( \langle \mathbf{h}, \mathbf{r}_R, \mathbf{t} \rangle + \langle \mathbf{h}, \mathbf{r}^{-1}_R, \mathbf{t} \rangle ) , $$

where $\langle \mathbf{h}, \mathbf{r}_R, \mathbf{t} \rangle = \sum_{i=1}^{d_f} \mathbf{h}[i] \cdot \mathbf{r}_R[i] \cdot \mathbf{t}[i]$ and each relation $R$ has two corresponding vectors, $\mathbf{r}_R, \mathbf{r}^{-1}_R \in \mathbb{R}^{d_f}$.

\subsection{Monotonically Increasing Neural Network Models} \label{app:monotonic_increasing_neural_network}
We now consider some neural network models.
When restricting their main parameters--- expressed as vectors, matrices, tensors, or neural network weights---to contain only non-negative values, and by making all non-linear functions in them monotonically increasing and non-negative (e.g., ReLU), we obtain monotonically increasing scoring functions.

SME is defined by:

$$ f(R, \mathbf{h}, \mathbf{t}) = g_u(\textbf{h}, \textbf{r}_R)^\top g_v(\textbf{t}, \textbf{r}_R)^\top , $$

where $g_u, g_v$ are neural networks, $k \in \mathbb{N}$, and each $\mathbf{r}_R \in \mathbb{R}^{k}$.
NTN is defined by:

$$ f(R, \mathbf{h}, \mathbf{t}) = \mathbf{r}_R^\top ~\sigma (\mathbf{h}^\top \mathbf{T}_R \mathbf{t} + \mathbf{M}_R^1 \mathbf{h} + \mathbf{M}_R^2 \mathbf{t} + \mathbf{b}_R) , $$

where $\sigma$ is a non-linear activation function, $k \in \mathbb{N}$, each $\mathbf{T}_R \in \mathbb{R}^{k \times d_f \times d_f}$, $\mathbf{M}_R^1, \mathbf{M}_R^2 \in \mathbb{R}^{k \times d_f}$, and $\mathbf{r}_R, \mathbf{b}_R \in \mathbb{R}^{k}$.
When every $\mathbf{T}_R$ and $\mathbf{b}_R$ is set to $\mathbf{0}$, it yields the model SLM.
MLP is defined as follows:

$$ f(R, \mathbf{h}, \mathbf{t}) = \mathbf{w}^\top \sigma (\mathbf{M}^1 \textbf{h} + \mathbf{M}^2 \textbf{r}_R + \mathbf{M}^3 \textbf{t}) $$

where $\sigma$ is a non-linear activation function, $k \in \mathbb{N}$, $\mathbf{M}^1, \mathbf{M}^2, \mathbf{M}^3 \in \mathbb{R}^{k \times d_f}$, $\mathbf{w} \in \mathbb{R}^{k}$, and each $\mathbf{r}_R \in \mathbb{R}^{d_f}$.

NAM is defined as follows:

$$ f(R, \mathbf{h}, \mathbf{t}) = \mathbf{t}^\top \cdot N(\mathbf{h}, \mathbf{r}_R) , $$

where $N$ is a deep neural network and each $\mathbf{r}_R \in \mathbb{R}^{d_f}$.
ConvE (Dettmers et al., 2018) is defined as follows:

$$f(R, \mathbf{h}, \mathbf{t}) = \sigma( \text{vec} (f([\bar{\mathbf{h}}; \bar{\mathbf{r}}] \star \omega )) \mathbf{W}) \mathbf{t} ,$$

where $\mathbf{r}_R$ is a vector, $\bar{\mathbf{h}}, \bar{\mathbf{r}}$ denote a 2D reshaping of $\mathbf{h}, \mathbf{r}$ respectively (i.e. a vector with dimension $k_1 \cdot k_2$ is turned into a matrix of dimension $k_1 \times k_2$), $[\_;\_]$ denotes concatenation, $\star~ \omega$ denotes a convolution with filters $\omega$, vec flattens an input tensor into a vector, $\sigma$ is a non-linear activation function, and $\mathbf{W}$ is a matrix.
HypER (Balazevic et al., 2019) is defines as follows, where the notation is the same as above, except additionally $\text{vec}^{-1}$ turns an input vector into a matrix and $\mathbf{H}$ is a matrix:

$$f(R, \mathbf{h}, \mathbf{t}) = \sigma( \text{vec} (\mathbf{h} \star \text{vec}^{-1} (\mathbf{w}_R \mathbf{H})) \mathbf{W}) \mathbf{t} .$$

\subsection{Initial Node Features in the Absence of Unary Predicates} \label{app:node_features_no_unary}
In scenarios where the KG completion task in non-inductive, i.e. where the test set contains only constants that are seen during training, the dominant approach is to learn an embedding for each constant.
However, in inductive settings such as those of \cite{teru2020inductive} and \cite{liu2023revisiting}, this approach does not work.
Furthermore, in these particular datasets (as well as others), there are no unary predicates, only binary predicates.
Thus, since the straightforward canonical encoding method uses unary predicates to obtain initial features for each node, an alternative is needed in such scenarios.
A simple approach, which we adopt since it does not change the model expressivity, is to have a single dummy unary predicate in the signature.
In this way, every node will have the same initial feature.

Other approaches to this problem include that of \cite{teru2020inductive}, where each node $w$ in the sub-graph around nodes $u$ and $v$ is labelled with the tuple $(d(w, u), d(w, v))$, where $d(w, u)$ denotes the shortest distance between $w$ and $u$ in the sub-graph.
\cite{hamaguchi2017knowledge} use a GNN to compute embeddings for constants that were not seen during training, use existing embeddings for neighbouring constants that appeared in training.
\cite{schlichtkrull2018modeling} state that ``the input to the first layer can be chosen as a unique one-hot vector for each node in the graph if no other features are present''.
\cite{berg2017graph} also use unique one-hot vectors for every node in the graph.
Finally, \cite[Section 4.1]{zhang2018link} give an initial feature to each node based on the structure of an enclosing sub-graph.
\section{Full Proofs}

\subsection{Lemma \ref{lemma:isomorphism}} \label{app:proof:lemma:isomorphism}

We introduce a definition and lemma that will aid in the proof of \Cref{prop:sound_rule_checking}.

\begin{definition}
An \emph{isomorphism} from a dataset $D$ to $D'$ is an injective mapping $h$ of constants to constants that is defined on $\texttt{con}(D)$ and satisfies $h(D) = D'$, where $h(D)$ is the dataset obtained by replacing each fact $U(s) \in D$ with $U(h(s))$ and each fact of the form $R(s, t)$ with $R(h(s), h(t))$.
\end{definition}

\begin{lemma} \label{lemma:isomorphism}
Let $\mathcal{N}$ be a max-sum GNN and $f$ a scoring function.
Then for all datasets $D, D'$, each isomorphism from $D$ to $D'$ is an isomorphism from $T_{\mathcal{N}, f}(D)$ to $T_{\mathcal{N}, f}(D')$.
\end{lemma}

\begin{proof}
Let $h$ be an isomorphism from $D$ to $D'$.
This induces a bijective mapping $\pi$ between the vertices of $\texttt{enc}(D)$ and $\texttt{enc}(D')$.

Let vertices $u, v \in \texttt{enc}(D)$ and $u' := \pi(u), v' := \pi(v) \in \texttt{enc}(D')$. Define $\mathbf{u}_L$ (and similarly $\mathbf{v}_L, \mathbf{u}_L', \mathbf{v}_L'$) to be the final embedding of $v$ after $\mathcal{N}$ has been applied.

Since the result of applying $\mathcal{N}$ depends only on the structure of the graph and not the specific names of the vertices, we have that $\mathbf{u}_L = \mathbf{u}_L'$ and  $\mathbf{v}_L = \mathbf{v}_L'$, and thus that for any binary relation $R$, $f(R, \mathbf{u}_L, \mathbf{v}_L) = f(R, \mathbf{u}_L', \mathbf{v}_L')$.

Thus, for any fact $R(s, t) \in D$, we must also have $R(h(s), h(t)) \in D'$.
\end{proof}

\subsection{Lemma \ref{lemma:monotonic_extension}} \label{app:proof:lemma:monotonic_extension}
\begin{lemma*}
Let $\mathcal{N}$ be a monotonic max-sum GNN and $f$ a monotonically increasing scoring function.
Then for all datasets $D, D'$ such that $D \subseteq D'$, $T_{\mathcal{N}, f}(D) \subseteq T_{\mathcal{N}, f}(D')$.
\end{lemma*}

\begin{proof}
Let $D, D'$ be datasets such that $D \subseteq D'$ and $E, E'$ their respective edges in the canonical encoding.
For $v \in \texttt{enc}(D)$, let $\mathbf{v}_\ell$ be the vector labelling of $v$ after $\ell$ layers of $\mathcal{N}$ applied to $D$.
Similarly define $\mathbf{v}_\ell'$ to be the labelling after $\ell$ layers of $\mathcal{N}$ applied to $D'$.

We prove by induction on $0 \leq \ell \leq L$ that $\mathbf{v}_\ell[i] \leq \mathbf{v}_\ell'[i]$ for each vertex $v \in \texttt{enc}(D)$ and $1 \leq i \leq \delta_\ell$.
The base case follows from the canonical encoding and the fact that $D \subseteq D'$.

For the inductive step, consider a layer $\ell$ and $1 \leq i \leq \delta_\ell$.
The canonical encoding and $D \subseteq D'$ implies that $E^c \subseteq E'^c$.
For every $u \in \texttt{enc}(D)$ and $1 \leq j \leq \delta_{\ell - 1}$, the inductive hypothesis gives us that $\mathbf{u}_{\ell - 1}[j] \leq \mathbf{u}_{\ell - 1}'[j]$.
Together, these imply that for every $c \in \text{Col}$

\begin{align*}
&\text{max-$k_\ell$-sum}(\multisetopen \mathbf{u}_{\ell - 1}[j] ~|~ (v, u) \in E^c \multisetclose) \\
\leq~ &\text{max-$k_\ell$-sum}(\multisetopen \mathbf{u}_{\ell - 1}'[j] ~|~ (v, u) \in E'^c \multisetclose) .
\end{align*}

Also, since all elements of $\mathbf{A}_\ell$ and each $\mathbf{B}_\ell^c$ are non-negative and $\sigma_\ell$ is monotonically increasing, we have $\mathbf{v}_\ell[i] \leq \mathbf{v}_\ell'[i]$.

Thus, as proven by induction, $\mathbf{v}_L[i] \leq \mathbf{v}_L'[i]$ for each vertex $v \in \texttt{enc}(D)$ and $1 \leq i \leq \delta_\ell$, and so $\mathbf{v}_L[i] \leq \mathbf{v}_L'[i]$ for all $i \in \{ 1, ..., \delta_L \}$.

Now let $R(s, t) \in T_{\mathcal{N}, f}(D)$ and let $u, v \in \texttt{enc}(D)$ be the corresponding nodes for $s, t$.
Then $f(R, \mathbf{u}_L, \mathbf{v}_L) \geq t_f$.
But as proven above, we then have $\mathbf{u}_L[i] \leq \mathbf{u}_L'[i]$ and $\mathbf{v}_L[i] \leq \mathbf{v}_L'[i]$ for all $i \in \{ 1, ..., \delta_L \}$.

Now since $f$ is monotonically increasing, we have $f(R, \mathbf{u}_L, \mathbf{v}_L) \leq f(R, \mathbf{u}_L', \mathbf{v}_L')$, and so $f(R, \mathbf{u}_L', \mathbf{v}_L') \geq t_f$.
So $R(s, t) \in T_{\mathcal{N}, f}(D')$, as required.
\end{proof}

\subsection{Proposition \ref{prop:sound_rule_checking}} \label{app:proof:prop:sound_rule_checking}
\begin{proposition*}
Let $\mathcal{N}$ be a monotonic max-sum GNN and $f$ a monotonically increasing scoring function.
Let $r$ be a Datalog rule with a binary head atom $H$, a (possibly empty) set $A$ of body atoms, and a (possibly empty) set $I$ of body inequalities.
For each variable $x$ in $r$, let $a_x, b_x$ be distinct constants uniquely associated with $x$.
Then, $r$ is sound for $(\mathcal{N}, f)$ if and only \changemarker{if} $H \mu \in T_{\mathcal{N}, f}(D_\mu)$ for each substitution $\mu$ mapping the variables of $r$ to constants in the set $\{ a_x ~|~ \text{$x$ is a variable of $r$} \}$ such that $\mu(x) \not= \mu(y)$ for each inequality $x \not\approx y \in I$, and
$D_\mu = A\mu \cup \{ R(b_x, \mu(x)) ~|~ \text{$x$ is a variable occurring in $H$ but not in $A$} \}$ where $R$ is an arbitrary but fixed binary predicate.
\end{proposition*}

\begin{proof}
Let $\mathcal{N}$, $r$, $H$, $A$, and $I$ be as specified in the proposition, and likewise $a_x, b_x$ for each variable $x$ in $r$.
We prove the claim in both directions.

\begin{center}
\textbf{Forward Direction ($\implies$)}
\end{center}

Assume $r$ is sound for $(\mathcal{N}, f)$ and consider arbitrary $\mu, D_\mu$ as specified in the proposition.
Note that $\mu$ is defined on the variables of $r$.
Whichever constants $\mu$ maps to from variables in the body of $r$ appear in $D_\mu$ by virtue of them being in $A \mu$.
For any variable $x$ in $H$ but not in $A$, $\mu(x)$ will also appear in $D_\mu$ by its definition.

Thus, $\mu$ is a substitution mapping from the variables of $r$ to the constants in $D_\mu$.
So by the properties of Datalog, we have $H \mu \in T_r(D_\mu)$.
Then since $r$ is sound, we obtain $H \mu \in T_\mathcal{N}(D_\mu)$, as required.

\begin{center}
\textbf{Backward Direction ($\impliedby$)}
\end{center}

Assume that $H \mu \in T_\mathcal{N}(D_\mu)$ for each $\mu$ as defined in the proposition.
To prove the soundness of $r$, consider an arbitrary dataset $D$: we prove that $T_r(D) \subseteq T_\mathcal{N}(D)$.
Consider an arbitrary fact in $T_r(D)$.
It will have the form $H \nu$ for some substitution $\nu$ from the variables of $r$ to the constants in $D$ such that $D \models B_i \nu$ for each literal $B_i$ in the body of $r$ (which may be an atom from $A$ or inequality from $I$).

We construct a mapping $h$ from constants in the range of $\nu$ to the set $\{ a_x ~|~ x \text{ is a variable of } r \}$.
For each constant $d$ in the range of $\nu$, define $h(d) = a_x$, where $x$ is a variable of $r$ such that $\nu(x) = d$ (if there are multiple such variables $x$, simply choose one).
Notice that $h$ is an injective mapping, since if $h(d_1) = h(d_2) = a_x$, then $d_1 = \nu(x) = d_2$, so $d_1 = d_2$.

We now let $\mu$ be the substitution $\mu(x) = h(\nu(x))$ for each variable $x$ of $r$, and let $D_\mu$ be defined as in the proposition.
Note that for all inequalities $B_i \in I$, $D \models B_i \nu$ and $h$ being injective mean that $D_\mu \models B_i\mu$, so $\mu$ satisfies the constraints on its definition given in the proposition.
Thus, $H \mu \in T_{\mathcal{N}, f}(D_\mu)$.

Define $D_\nu = A\nu \cup \{ R(b_x, \nu(x)) ~|~ x$ is a variable occurring in $H$ but not in $A \}$ and $D' = D \cup \{ R(b_x, \nu(x)) ~|~ \text{$x$ is a variable occurring in $H$ but not in $A$} \}$, where $R$ is as defined in the proposition.
Then $h(D_\nu) = D_\mu$, so $h$ is an isomorphism from $D_\nu$ to $D_\mu$.
By \Cref{lemma:isomorphism}, $h$ is thus an isomorphism from $T_{\mathcal{N}, f}(D_\nu)$ to $T_{\mathcal{N}, f}(D_\mu)$, and $h(T_{\mathcal{N}, f}(D_\nu)) = T_{\mathcal{N}, f}(D_\mu)$.

Since we have $H \mu \in T_{\mathcal{N}, f}(D_\mu)$, we must then have $h^{-1}(H \mu) \in T_{\mathcal{N}, f}(D_\nu)$, so $H \nu \in T_{\mathcal{N}, f}(D_\nu)$.
But $A\nu \subseteq D$, so $D_\nu \subseteq D'$.
Then by \Cref{lemma:monotonic_extension}, $T_{\mathcal{N}, f}(D_\nu) \subseteq T_{\mathcal{N}, f}(D')$, so $H \nu \in T_{\mathcal{N}, f}(D')$.

Finally, note that the possible extra (binary) facts in $D'$ that are not in $D$ connect constants not in $D$ to constants in $D$, but not the other way round; thus, the direction of GNN message propagation (reverse of the edge direction) ensures that $T_{\mathcal{N}, f}(D')$ and $T_{\mathcal{N}, f}(D)$ coincide on facts derived over the constants of $D$. Since all constants in the range of $\nu$ are in $D$, we have $H \nu \in T_{\mathcal{N}, f}(D)$, as required.

\end{proof}

\subsection{Proposition \ref{prop:rule_captured_effect}} \label{app:proof:prop:rule_captured_effect}
\begin{proposition*}
Let $f$ be a scoring function and $B \rightarrow H$ a safe rule such that for all constants $C$ and $e: C \to \mathbb{R}^{d_f}$, $(f, C, e)$ captures $B \rightarrow H$.
Let $\mathcal{N}$ be a GNN.
Assume that for each atom $B_i$ of $B$, there is an inequality-free rule $A_i \rightarrow B_i$ that is sound for $(\mathcal{N}, f)$.
Then the rule $\bigwedge_i A_i \rightarrow H$ is sound for $(\mathcal{N}, f)$.
\end{proposition*}

\begin{proof}

Assume that for all atoms $B_i$ in $B$, rule $r_2^i = A_i \rightarrow B_i$ is sound for $(\mathcal{N}, f)$ and let $r_3 = \bigwedge_i A_i \rightarrow H$.
We prove that for all datasets $D$, $T_{r_3}(D) \subseteq T_{\mathcal{N}, f}(D)$.
To this end, consider a fact in $T_{r_3}(D)$.
It will have the form $H \nu$ for some substitution $\nu$ from the variables of ${r_3}$ to the constants of $D$ such that $D \models \bigwedge_i A_i \nu$.
But then for each $i$, we have $D \models A_i \nu$, and thus $B_i \nu \in T_{r_2^i}(D)$.

We now define $C := \texttt{con}(D)$ and $e: C \to \mathbb{R}^{d_f}$ by $e(c) = \mathbf{v}^c_L$, the final embedding of the vertex corresponding to the constant $c$ after $\mathcal{N}$ is applied to $D$.
We likewise define $\mathbf{v}^a_L, \mathbf{v}^b_L$ for constants $a, b$.

Now consider an arbitrary atom $B_i$ in $B$.
Note that $B_i\nu = R_1(a,b)$, for some constants $a,b$ in the range of $\nu$ and binary predicate $R_1$.
From the soundness of $r_2^i$ and $B_i \nu \in T_{r_2^i}(D)$, we thus obtain $R_1(a,b) \in T_{\mathcal{N}, f}(D)$.

Then since $R_1(a,b) \in T_{\mathcal{N}, f}(D)$, we have $f(R_1, \mathbf{v}^a_L, \mathbf{v}^b_L) \geq t_f$, which implies $f(R_1, e(a), e(b)) \geq t_f$, and thus $R_1(a,b) \in \mathcal{F}_f(C,e)$.

Thus for every fact $B_i \nu$ in $B \nu$, we have $\mathcal{F}_f(C,e) \models B_i \nu$.
So by definition, $\mathcal{F}_f(C,e) \models B \nu$.
But as assumed in the lemma, $(f, C, e)$ captures $B \rightarrow H$.
So we obtain that $\mathcal{F}_f(C,e) \models H \nu$.

Note that $H\nu = R_2(c,d)$, for some constants $c,d$ in the range of $\nu$ and binary predicate $R_2$.
Going through the computation again, this means that $\mathcal{F}_f(C,e) \models R_2(c,d)$, so $f(R_2, e(c), e(d)) \geq t_f$, thus $f(R_2, \mathbf{v}^c_L, \mathbf{v}^d_L) \geq t_f$, which implies $R_2(c, d) \in T_{\mathcal{N}, f}(D)$.
\end{proof}

\subsection{Corollary \ref{cor:rule_pattern_conjunction}} \label{app:proof:cor:rule_pattern_conjunction}
\begin{corollary*}
Let $F$ be a scoring function family that universally captures a rule pattern $\rho$, $B \rightarrow H$ a rule conforming to $\rho$, $\mathcal{N}$ a GNN, and $f \in F$.
Assume that for each atom $B_i$ of $B$, there is an inequality-free rule $A_i \rightarrow B_i$ sound for $(\mathcal{N}, f)$.
Then the rule $\bigwedge_i A_i \rightarrow H$ is sound for $(\mathcal{N}, f)$.
\end{corollary*}
\begin{proof}

$F$ universally captures $\rho$.
So for all constants $C$ and $e: C \to \mathbb{R}^{d_f}$, we have $(f, C, e)$ captures $B \rightarrow H$, which is a rule conforming to $\rho$.
Then from \Cref{prop:rule_captured_effect}, we have that $\bigwedge_i A_i \rightarrow H$ is sound for $(\mathcal{N}, f)$.
\end{proof}

\subsection{Corollary \ref{cor:rule_pattern_exactly}} \label{app:proof:cor:rule_pattern_exactly}
\begin{corollary*}
Let $F$ be a scoring function family that exactly captures a rule pattern $\rho$, and $B \rightarrow H$ a rule conforming to $\rho$.
Then there exists $f \in F$ such that: for each GNN $\mathcal{N}$ and inequality-free rules $A_i \rightarrow B_i$ sound for $(\mathcal{N}, f)$ for   each $B_i$ in $B$, the rule $\bigwedge_i A_i \rightarrow H$ is sound for $(\mathcal{N}, f)$.
\end{corollary*}

\begin{proof}
$F$ exactly captures $\rho$, so there exists $f \in F$ such that for all constants $C$ and $e: C \to \mathbb{R}^{d_f}$, $(f, C, e)$ captures $B \rightarrow H$, a rule conforming to $\rho$.
So let $\mathcal{N}$ be a GNN with $A_i \rightarrow B_i$ sound for $(\mathcal{N}, f)$ for all atoms $B_i$ of $B$ (where each $A_i$ is a conjunction of atoms).
Then from \Cref{prop:rule_captured_effect}, we have that $\bigwedge_i A_i \rightarrow H$ is sound for $(\mathcal{N}, f)$.
\end{proof}

\subsection{Theorem \ref{thm:max_rule_shape}} \label{app:proof:thm:max_rule_shape}
\begin{theorem*}
Let $\mathcal{N}$ be a monotonic max GNN and $f$ a monotonically increasing scoring function.
Let $\mathcal{P}_\mathcal{N}$ be the Datalog program containing, up to variable renaming, each $(L, |\text{Col}| \cdot \delta_\mathcal{N})$-tree-like rule without inequalities that is sound for $(\mathcal{N}, f)$, where $\delta_\mathcal{N} = \text{max}(\delta_0, ..., \delta_L)$. Then $T_{\mathcal{N}, f}$ and $\mathcal{P}_\mathcal{N}$ are equivalent.
\end{theorem*}

\begin{proof}
We show that $T_{\mathcal{N}, f}(D) = T_{P_\mathcal{N}}(D)$ holds for every dataset $D$.
By definition of $P_\mathcal{N}$, each rule in $P_\mathcal{N}$ is sound for $(\mathcal{N}, f)$, so we have $T_r(D) \subseteq T_{\mathcal{N}, f}(D)$ for every $r \in P_\mathcal{N}$, and thus $T_{P_\mathcal{N}}(D) \subseteq T_{\mathcal{N}, f}(D)$.
So we prove $T_{\mathcal{N}, f}(D) \subseteq T_{P_\mathcal{N}}(D)$.

Let $\alpha$ be an arbitrary binary fact in $T_{\mathcal{N}, f}(D)$.
Observe that $\alpha$ can be of the form $R(s, t)$, for distinct constants $s, t$, or of the form $R(s, s)$, for a constant $s$.
We will denote these two cases as \textbf{(Case 1)} and \textbf{(Case 2)}, respectively. We will prove both cases simultaneously, making a separate argument when each case needs to be treated differently.

We construct a $(L, |\text{Col}| \cdot \delta_\mathcal{N})$-tree-like rule without inequalities $r$, such that $\alpha \in T_r(D)$ and $r$ is sound for $(\mathcal{N}, f)$.
Together, these will imply that $\alpha \in T_{P_\mathcal{N}}(D)$, as required for the proof.

Let $G = (V, E, \lambda)$ be the canonical encoding of $D$ and let $\lambda_0, ..., \lambda_L$ the vertex labeling functions arising from the application of $\mathcal{N}$ to $G$.

\begin{center}
\textbf{Setup of Construction}
\end{center}

We now construct a binary atom $H$, a conjunction $\Gamma$, a substitution $\nu$ from the variables of $\Gamma$ to the set of constants in $D$, and a graph $U$ (without vertex labels) where each vertex in $U$ is of the form $u^x$ for $x$ a variable and each edge has a colour in Col.
We also construct mappings $M_{c,\ell,j}: U \to V$ for each $c \in \text{Col}$, $1 \leq \ell \leq L$, and $1 \leq j \leq \delta_{\ell - 1}$.
Intuitively, $M_{c,\ell,j}(u^x)$ represents the $c$-coloured neighbour of $v_t$ that contributes to the result of the max aggregation in vector index $j$ at layer $\ell$ (where $v_t$ is the vertex corresponding to the constant $\nu(x)$).
We initialise $\Gamma$ as the empty conjunction, and we initialise $\nu$ and each $M_{c,\ell,j}$ as empty mappings.

We will assign to each vertex in $U$ a \emph{level} between $0$ and $L$. We will identify a single vertex from $U$ as the \emph{first root} vertex. Furthermore, if $\alpha$ is of the form $R(s, t)$, then we identify another distinct vertex from $U$ as the \emph{second root} vertex.

In the remainder of the proof, we use letters $s, t$ for constants in $D$, letters $x, y$ for variables, letters $v, w$ for the vertices in $V$, and (possibly indexed) letter $u$ for the vertices in $U$.
Our construction is by induction from level $L$ down to level $1$. The base case defines one or two vertices of level $L$. Then, for each
$1 \leq \ell \leq L$, the induction step considers the vertices of level $\ell$ and defines new vertices of level $\ell - 1$.

\begin{center}
\textbf{Construction: Base Case}
\end{center}

\textbf{(Case 1)} If $\alpha$ is of the form $R(s, t)$ for a binary predicate $R$, then $V$ contains vertices $v_s, v_t$.
We introduce fresh variables $x, y$, and define $\nu(x) = s, \nu(y) = t$;
we define $H = R(x, y)$;
we introduce vertices $u^x, u^y$ of level $L$;
and we make $u^x$ the first root vertex and $u^y$ the second root vertex.
Finally, we extend $\Gamma$ with atom $U(x)$ for each $U(s) \in D$, and also extend $\Gamma$ with atom $U(y)$ for each $U(t) \in D$.

\textbf{(Case 2)} If on the other hand, $\alpha$ is of the form $R(s, s)$, then $V$ contains a vertex $v^s$.
We introduce fresh variables $x, y$, and define $\nu(x) = s, \nu(y) = s$;
we define $H = R(x, y)$;
we introduce vertices $u^x, u^y$ of level $L$;
and we make $u^x$ the first root vertex and $u^y$ the second root vertex.
Finally, we extend $\Gamma$ with atoms $U(x)$ and $U(y)$ for each $U(s) \in D$.

\begin{center}
\textbf{Construction: Induction Step}
\end{center}

For the induction step, consider $1 \leq \ell \leq L$ and assume that all vertices of level greater than or equal to $\ell$ have been
already defined.
We then consider each vertex of the form $u^x$ of level $\ell$. Let $t = \nu(x)$.
For each colour $c \in \text{Col}$, each layer $1 \leq \ell' < \ell$, and each dimension $j \in \{ 1, ..., \delta_{\ell' - 1} \}$ for which there exists some $(v^t, w) \in E^c$, let

\begin{align*}
&M_{c, \ell', j}(u^x) = w, \text{for } (v^t, w) \in E^c \text{ and } \mathbf{w}_{\lambda_{\ell'}}[j] =\\
&\text{max}(\multisetopen \mathbf{w}_{\lambda_{\ell'}}[j] ~|~ (v^t, w) \in E^c \multisetclose)
\end{align*}

Intuitively, $w$ picks out the maximum $c$-coloured neighbour of $v^t$ in feature vector index $j$ at layer $\ell'$.
Note that there may be multiple such $w$ that satisfy the above equation: any of them may be chosen.
Note that such a $w$ must exist, since we assumed that there exists some $(v^t, w) \in E^c$.
Also, $w$ must be of the form $v^s$ for some constant $s$ in $D$.

We then introduce a fresh variable $y$ and define $\nu(y) = s$.
We introduce a vertex $u^y$ of level $\ell - 1$ and an edge $E^c(u^x, u^y)$ to $U$.
Finally, we append to $\Gamma$ the conjunction

\begin{align}
E^c(x, y) ~\land~ ~\bigwedge_{U(s) \in D} U(y) ,
\end{align}

where $v^s = w = M_{c, \ell', j}(u^x)$ and $U$ stands in for arbitrary unary predicates.
This step adds at most $|\text{Col}| \cdot \delta_{\ell' - 1} \cdot \ell$ new successors of $u^x$.
This completes the inductive construction.

\begin{center}
\textbf{Rule Shape and Immediate Consequences}
\end{center}

At this point in the proof, $\Gamma$ consists of a conjunction of two formulas, one of which is $(L, |\text{Col}| \cdot \delta_\mathcal{N})$-tree-like for $x$ and the other of which is $(L, |\text{Col}| \cdot \delta_\mathcal{N})$-tree-like for $y$ (where $x, y$ were defined in the base case).
We also have $H = R(x, y)$, so $\Gamma \rightarrow H$ is a $(L, |\text{Col}| \cdot \delta_\mathcal{N})$-tree-like rule.
We define this rule to be $r$.

Furthermore, the construction of $\nu$ means $D \models \Gamma \nu$.
So we have $H \nu \in T_r(D)$, with $H \nu = \alpha$, so $\alpha \in T_r(D)$, as required.

Note that if $s, t$ are such that they are not mentioned in any fact in $D$, then $\Gamma$ is empty, and $r$ will then have the form $\top \rightarrow R(x, y)$.
Under our definitions, $r$ is still a $(L, |\text{Col}| \cdot \delta_\mathcal{N})$-tree-like rule in this case, since $\top$ is tree-like for $x$ and $y$.
This applies similarly if only one of $s, t$ is not mentioned in any fact in $D$.

\begin{center}
\textbf{Rule Soundness}
\end{center}

To complete the proof, we show that $r$ is sound for $(\mathcal{N}, f)$.
For this purpose, let $D'$ be an arbitrary dataset and $\alpha'$ an arbitrary ground atom such that $\alpha' \in T_r(D')$.
Then there exists a substitution $\nu'$ such that $D' \models \Gamma \nu'$.

Let the graph $G = (V', E', \lambda')$ be the canonical encoding of $D'$ and $\lambda_0', ..., \lambda_L'$ the functions labelling the vertices of $G'$ when $\mathcal{N}$ is applied to it.
We will use letters $p, q, q'$ to denote vertices of $V'$.

We prove the following statement by induction, which relies on the monotonicity of $\mathcal{N}$: for each $0 \leq \ell \leq L$ and each vertex $u^x$ of $U$ whose level is at least $\ell$, we have $\mathbf{v}_{\lambda_\ell}[i] \leq \mathbf{p}_{\lambda_\ell'}[i]$ for each $i \in \{ 1, ..., \delta_\ell \}$, where $v = v^{\nu(x)}$ and $p = v^{\nu'(x)}$.

\begin{center}
\textbf{Monotonicity: Base Case}
\end{center}

For the base case, $\ell = 0$, consider an arbitrary $1 \leq i \leq \delta_0$ and $u^x \in U$.
Let $v = v^{\nu(x)}$ and $p = v^{\nu'(x)}$.
Since $\ell = 0$, values $\mathbf{v}_{\lambda_0}[i], \mathbf{p}_{\lambda_0'}[i]$ come from the canonical encoding, so $\mathbf{v}_{\lambda_0}[i] \in \{ 0, 1 \}$ and $\mathbf{p}_{\lambda_0'}[i] \in \{ 0, 1 \}$.

Thus, we prove that $\mathbf{v}_{\lambda_0}[i] = 1$ implies $\mathbf{p}_{\lambda_0'}[i] = 1$.
By definition of the canonical encoding, $\mathbf{v}_{\lambda_0}[i] = 1$ implies $U_i(x \nu) \in D$.
From the construction of $\Gamma$, we have $U_i(x) \in \Gamma$.
Also, by definition of $\nu'$, we have $D' \models \Gamma \nu'$.
So $U_i(x \nu') \in D'$ and $\mathbf{p}_{\lambda_0'}[i] = 1$, as required.

\begin{center}
\textbf{Monotonicity: Induction Step}
\end{center}

For the induction step, assume that the property holds for some $\ell - 1$.
Consider an arbitrary vertex $u^x \in U$ with level at least $\ell$, $c \in \text{Col}$, and $i \in \{ 1, ..., \delta_i \}$.
Let $v = v^{\nu(x)}$ and $p = v^{\nu'(x)}$.
Then $\mathbf{v}_{\lambda_\ell}[i], \mathbf{p}_{\lambda_\ell'}[i]$ are computed as follows:

\begin{align} \label{eq:v_max_update}
&\mathbf{v}_{\lambda_\ell}[i] = \sigma_\ell \biggl(
\mathbf{b}_\ell[i] +
\sum_{j=1}^{\delta_{\ell-1}}
\mathbf{A}_\ell[i, j] \mathbf{v}_{\lambda_{\ell - 1}}[j] + \\
&\sum_{c \in \text{Col}} \sum_{j=1}^{\delta_{\ell - 1}}
\mathbf{B}_\ell[i, j]
\text{max} \multisetopen
\textbf{w}_{\lambda_{\ell - 1}}[j] ~|~ (v, w) \in E^c
\multisetclose
\biggr)
\end{align}

\begin{align} \label{eq:p_max_update}
&\mathbf{p}_{\lambda_\ell'}[i] = \sigma_\ell \biggl(
\mathbf{b}_\ell[i] +
\sum_{j=1}^{\delta_{\ell-1}}
\mathbf{A}_\ell[i, j] \mathbf{p}_{\lambda_{\ell - 1}'}[j] + \\
&\sum_{c \in \text{Col}} \sum_{j=1}^{\delta_{\ell - 1}}
\mathbf{B}_\ell[i, j]
\text{max} \multisetopen
\textbf{q}'_{\lambda'_{\ell - 1}}[j] ~|~ (p, q') \in E'^c
\multisetclose
\biggr)
\end{align}

The induction assumption ensures $\mathbf{v}_{\lambda_{\ell - 1}}[j] \leq \mathbf{p}_{\lambda_{\ell - 1}'}[j]$ for each $1 \leq j \leq \delta_{\ell-1}$.
Also for each $c \in \text{Col}$ and $1 \leq j \leq \delta_{\ell - 1}$, we have $\text{max} \multisetopen \textbf{w}_{\lambda_{\ell - 1}}[j] ~|~ (v, w) \in E^c \multisetclose = \mathbf{w}_{\lambda_{\ell - 1}}[j]$, where $w = M_{c, \ell - 1, j}(u^x)$.

If $w = M_{c, \ell - 1, j}(u^x)$ is undefined, then there does not exists a $(v^t, w) \in E^c$, and so $\text{max} \multisetopen \textbf{w}_{\lambda_{\ell - 1}}[j] ~|~ (v, w) \in E^c \multisetclose = 0$.
Then $\mathbf{w}_{\lambda_{\ell - 1}}[j] \leq \text{max} \multisetopen \textbf{q}'_{\lambda'_{\ell - 1}}[j] ~|~ (p, q') \in E'^c \multisetclose$ holds trivially, since node values are non-negative.
So instead assume that $w = M_{c, \ell - 1, j}(u^x)$ is defined.

Recall that $w$ has the form $v^s$, where $s$ is a constant in $D$.
Furthermore, by construction of $U$, there is a vertex $u^y$ in $U$ of level $\ell - 1$ such that $\nu(y) = s$ and $E^c(u^x, u^y)$ is in $U$.

Furthermore, $\Gamma$ contains the atom $E^c(x, y)$.
We then have $E^c(\nu'(x), \nu'(y)) \in D'$.
Thus, $q = v^{\nu'(y)}$ is a $c$-neighbour of $v^{\nu'(x)}$ in $G'$.
The induction assumption means that for $w = v^{\nu(y)}$ and $q = v^{\nu'(y)}$, we have $\mathbf{w}_{\lambda_{\ell - 1}}[j] \leq \mathbf{q}_{\lambda'_{\ell - 1}}[j]$.
So $\text{max} \multisetopen \textbf{w}_{\lambda_{\ell - 1}}[j] ~|~ (v, w) \in E^c \multisetclose \leq \text{max} \multisetopen \textbf{q}'_{\lambda'_{\ell - 1}}[j] ~|~ (p, q') \in E'^c \multisetclose$.

Thus, by the computations in \Cref{eq:v_max_update} and \Cref{eq:p_max_update}, and since all elements of $\mathbf{A}_\ell$ and each $\mathbf{B}_\ell^c$ are non-negative and $\sigma_\ell$ is monotonically increasing, we have $\mathbf{v}_{\lambda_\ell}[i] \leq \mathbf{p}_{\lambda_\ell'}[i]$, as required.

\begin{center}
\textbf{Conclusion}
\end{center}

Recall that $\alpha' \in T_r(D')$, so $\alpha'$ has the form $R(s', t')$ with $s' = \nu'(x)$ and $t' = \nu'(y)$.

\textbf{(Case 1)} if $\alpha$ has the form $R(s, t)$, then $R(s, t) \in T_{\mathcal{N}, f}(D)$ for $s = \nu(x)$ and $t = \nu(y)$, since $\alpha = R(s, t)$ is the original binary fact we assumed that the model produces from $D$.

Now let $v = v^s, w = v^t, p = v^{s'}, q = v^{t'}$.
Then $R(s, t) \in T_{\mathcal{N}, f}(D)$ implies $f(R, \mathbf{v}_{\lambda_L}, \mathbf{w}_{\lambda_L}) \geq t_f$.

The above property ensures that $\mathbf{v}_{\lambda_L}[i] \leq \mathbf{p}_{\lambda_L'}[i]$ and $\mathbf{w}_{\lambda_L}[i] \leq \mathbf{q}_{\lambda_L'}[i]$ for all $i \in \{ 1, ..., \delta_L \}$.
Thus, since $f$ is monotonically increasing, we have $f(R, \mathbf{v}_{\lambda_L}, \mathbf{w}_{\lambda_L}) \leq f(R, \mathbf{p}_{\lambda_L'}, \mathbf{q}_{\lambda_L'})$.

So $f(R, \mathbf{p}_{\lambda_L'}, \mathbf{q}_{\lambda_L'}) \geq t_f$, which implies that $R(s', t') \in T_{\mathcal{N}, f}(D')$, as required to show the soundness of the rule.

\textbf{(Case 2)} if on the other hand, $\alpha$ has the form $R(s, s)$, then $R(s, s) \in T_{\mathcal{N}, f}(D)$ for $s = \nu(x)$, since $\alpha = R(s, s)$ is the original binary fact we assumed that the model produces from $D$.

Now let $v = v^s, p = v^{s'}, q = v^{t'}$.
Then $R(s, s) \in T_{\mathcal{N}, f}(D)$ implies that $f(R, \mathbf{v}_{\lambda_L}, \mathbf{v}_{\lambda_L}) \geq t_f$.
Recall that $\nu(x) = \nu(y) = s$, so we have $v = v^{\nu(x)} = v^{\nu(y)}$, $p = v^{\nu'(x)}$, and $q = v^{\nu'(y)}$.

Thus, the above property ensures that $\mathbf{v}_{\lambda_L}[i] \leq \mathbf{p}_{\lambda_L'}[i]$ and $\mathbf{v}_{\lambda_L}[i] \leq \mathbf{q}_{\lambda_L'}[i]$ for all $i \in \{ 1, ..., \delta_L \}$.
Thus, since $f$ is monotonically increasing, we have $f(R, \mathbf{v}_{\lambda_L}, \mathbf{v}_{\lambda_L}) \leq f(R, \mathbf{p}_{\lambda_L'}, \mathbf{q}_{\lambda_L'})$.

So $f(R, \mathbf{p}_{\lambda_L'}, \mathbf{q}_{\lambda_L'}) \geq t_f$, which implies that $R(s', t') \in T_{\mathcal{N}, f}(D')$, as required to show the soundness of the rule.

\end{proof}

\subsection{Capacity Setup} \label{app:capacity_setup}
For the following proofs, we cite from the paper ``From Monotonic Graph Neural Networks to Datalog and Back: Expressive Power and Practical Applications'' \cite{tena2025expressive}, which has not yet been published.

To start, we adopt the definition of the sets $\mathcal{X}_{\ell, i}$ as given in \changemarker{\cite[Definition 13]{tena2025expressive}}.
We compute ``capacity'' in \Cref{alg:gnn_capacity}, which is adapted from \changemarker{\cite[Algorithm 3]{tena2025expressive}} by changing only the first line.

\begin{algorithm}
\caption{Computing Capacity} \label{alg:gnn_capacity}

\Input{A monotonic max-sum GNN $\mathcal{N}$ and a non-negative bilinear scoring function $f$}
% \Output{}
Let $\alpha_L := \alpha$, as defined in \Cref{def:bilinear_capacity} \;
\For{$\ell$ from $L$ down to $1$}
{
\If{all elements of $\mathbf{A}_\ell$ and all $\mathbf{B}_\ell^c$ are $0$ or $\bigcup_i \mathcal{X}_{\ell-1,i} = \{ 0 \}$}
{
$\mathcal{C}_\ell := \mathcal{C}_{\ell-1} := \ldots := \mathcal{C}_1 := 0$ \;
\Return
}

$w_\ell := $ the least non-zero element of $\mathbf{A}_\ell$ and all $\mathbf{B}_\ell^c$ \;
$\epsilon_\ell := $ the least non-zero element of $\bigcup_i \mathcal{X}_{\ell-1,i}$ \;
$\beta_\ell := $ the least natural number such that $\sigma_\ell(\beta_\ell) \geq \alpha_\ell$ \;
$b_\ell := $ the least element of $\mathbf{b}_\ell$ \;
$\mathcal{C}_\ell := \text{min}(k_\ell, \text{max}(0, \lceil \frac{\beta_\ell - b_\ell}{w_\ell \cdot \epsilon_\ell} \rceil)$ \;
$\alpha_{\ell - 1} := \frac{\beta_\ell - b_\ell}{w_\ell}$ \;

}
\end{algorithm}
We then define the capacity of a max-sum GNN $\mathcal{N}$ and bilinear scoring function $f$.
\begin{definition}
The \emph{capacity} $\mathcal{C}_\ell$ of each layer $\ell$ of $\mathcal{N}$ is defined in \Cref{alg:gnn_capacity}. The \emph{capacity} of $(\mathcal{N}, f)$ is defined as $\mathcal{C}_{\mathcal{N}, f} = \text{max}\{\mathcal{C}_1, ..., \mathcal{C}_L\}$.
\end{definition}

\Cref{alg:gnn_capacity} makes reference to a new definition we provide, \Cref{def:bilinear_capacity}, which yields a value $\alpha$ and is presented again below.

\begin{definition*}
Consider a monotonic max-sum GNN $\mathcal{N}$ and a non-negative bilinear scoring function $f$.
For each binary relation $R$ in the signature, where $\mathbf{M}_R$ is its relation matrix in $f$, we define: $\alpha_R := 1$ if all elements of $\mathbf{M}_R$ are $0$ or $\bigcup_i \mathcal{X}_{L,i} = \{ 0 \}$.
Otherwise, we define $\alpha_R := $ the least natural number such that $\alpha_R \cdot w \cdot \epsilon \geq t_f$, where $w$ is the least non-zero element of $\mathbf{M}_R$ and $\epsilon$ is the least non-zero element of $\bigcup_i \mathcal{X}_{L,i}$.
Then, we define $\alpha := \text{max}\{ \alpha_R ~|~ \text{for each binary relation $R$ in the signature} \}$.
\end{definition*}

Intuitively, $\alpha$ represents a value such that once a vertex feature (output of a GNN) is already greater than it, increasing it further (by adding facts to the input dataset) cannot change the output of $T_{\mathcal{N}, f}$.

Note that \changemarker{\cite[Algorithm 4]{tena2025expressive}} can be used as defined to compute the smallest elements of each $\mathcal{X}_{\ell,i}$, which justifies the use of the term ``algorithm'' for \Cref{alg:gnn_capacity} and shows that $\alpha$ from \Cref{def:bilinear_capacity} can be computed.
\changemarker{\cite[Theorem 18]{tena2025expressive}} proves that \changemarker{\cite[Algorithm 4]{tena2025expressive}} is correct and terminates on all inputs.

During the proof of \changemarker{\cite[Theorem 17]{tena2025expressive}}, the following claim is proved, pertaining to max-sum GNNs whose aggregation has been restricted based on their capacity.

\begin{proposition} \label{prop:max_sum_lower_bound}
Let $D$ be a dataset, $\mathcal{N}$ a max-sum GNN with unbounded activation functions, and $\mathcal{N}'$ the GNN obtained from $\mathcal{N}$ by replacing $k_\ell$ with the capacity $\mathcal{C}_\ell$ for each $\ell \in \{1, ..., L\}$.
For vertex $v \in \texttt{enc}(D)$ and layer $\ell$, we define $\mathbf{v}_\ell, \mathbf{v}_\ell'$ to be the labelling of the vertices after layer $\ell$ by applying $\mathcal{N}, \mathcal{N'}$ to $\texttt{enc}(D)$, respectively.
Let $\ell_{\text{st}}$ be the largest $\ell \in \{ 1, ..., L \}$ such that either all elements of $\mathbf{A}_\ell$ and all $\mathbf{B}_\ell^c$ are $0$, or $\bigcup_i \mathcal{X}_{\ell-1,i} = \{ 0 \}$.
If such an $\ell$ does not exist,let  $\ell_{\text{st}} = 0$.

Then for each vertex $v \in \texttt{enc}(D)$, layer $\ell \in \{ \ell_{\text{st}}, ..., L \}$, and position $i \in \{ 1, ..., \delta_\ell \}$, either $\mathbf{v}_\ell[i] = \mathbf{v}_\ell'[i]$ or $\mathbf{v}_\ell[i] > \mathbf{v}_\ell'[i] \geq \alpha_\ell$.
\end{proposition}

Notice that in \changemarker{\cite[Algorithm 3]{tena2025expressive}}, $\alpha_L$ is defined in terms of the classifier of $\mathcal{N}$, and nowhere in the proof \changemarker{\cite{tena2025expressive}} of \Cref{prop:max_sum_lower_bound} is the classifier referred to.
Thus, for the purposes of the proposition, the definition of $\alpha_L$ is arbitrary and the result can be applied to \Cref{alg:gnn_capacity}.
Notice also that since $L \geq \ell_{\text{st}}$ always holds, the result can be applied to each $\mathbf{v}_L, \mathbf{v}_L'$.

Using this result, we prove that the number of neighbours used in each aggregation step can be upper-bounded without changing the consequences on any dataset.

\subsection{Theorem \ref{thm:capacity_maintains_equality}} \label{app:proof:thm:capacity_maintains_equality}
\begin{theorem*}
Let $\mathcal{N}$ be a monotonic max-sum GNN with unbounded activation functions and $f$ a non-negative bilinear scoring function.
Let $\mathcal{N}'$ be the GNN obtained from $\mathcal{N}$ by replacing $k_\ell$ with the capacity $\mathcal{C}_\ell$ for each $\ell \in \{1, ..., L\}$.
Then for each dataset $D$, it holds that $T_{\mathcal{N}, f}(D) = T_{\mathcal{N}', f}(D)$.
\end{theorem*}

\begin{proof}
Let $D$ be a dataset, $a, b \in D$ be constants (where possibly $a = b$ or $a \not= b$), and $R$ a binary predicate.
Let $v^a, v^b \in \texttt{enc}(D)$ be the vertices corresponding to $a, b$ respectively.
For a vertex $v \in \texttt{enc}(D)$, let $\mathbf{v}_L, \mathbf{v}_L'$ be the labelling of the vertices after layer $L$ by applying $\mathcal{N}, \mathcal{N'}$ to $\texttt{enc}(D)$, respectively.

To prove the required equality, we show that $R(a,b) \in T_{\mathcal{N}, f}(D) \iff R(a,b) \in T_{\mathcal{N}', f}(D)$.
This is equivalent to proving that $f(R, {\mathbf{v}_L^a}, {\mathbf{v}_L^b}) \geq t_f \iff f(R, {\mathbf{v}_L^a}', {\mathbf{v}_L^b}') \geq t_f$, which in turn is equivalent to ${\mathbf{v}_L^a} \mathbf{M}_R {\mathbf{v}_L^b} \geq t_f \iff {\mathbf{v}_L^a}' \mathbf{M}_R {\mathbf{v}_L^b}' \geq t_f$, where $\mathbf{M}_R$ is the relation matrix of $R$ in $f$.

\begin{center}
\textbf{Backward Direction ($\impliedby$)}
\end{center}

As proven at the start of the proof of \changemarker{\cite[Theorem 17]{tena2025expressive}}, $\forall i \in \{ 1, ..., \delta_L \}$ it holds that ${\mathbf{v}_L^a}[i] \geq {\mathbf{v}_L^a}'[i]$ and ${\mathbf{v}_L^b}[i] \geq {\mathbf{v}_L^b}'[i]$.
The proof of this follows trivially from the definition of max-sum aggregation.

Thus, from the monotonicity of $f$, we have that $f(R, {\mathbf{v}_L^a}, {\mathbf{v}_L^b}) \geq f(R, {\mathbf{v}_L^a}', {\mathbf{v}_L^b}')$.

So if $f(R, {\mathbf{v}_L^a}', {\mathbf{v}_L^b}') \geq t_f$, then also $f(R, {\mathbf{v}_L^a}, {\mathbf{v}_L^b}) \geq t_f$.

\begin{center}
\textbf{Forward Direction ($\implies$)}
\end{center}

Assume that ${\mathbf{v}_L^a} \mathbf{M}_R {\mathbf{v}_L^b} \geq t_f$.
First, consider the definition of $\alpha_R$ in \Cref{def:bilinear_capacity}.
If all elements of $\mathbf{M}_R$ are $0$ or $\bigcup_i \mathcal{X}_{L,i} = \{ 0 \}$, then ${\mathbf{v}_L^a} \mathbf{M}_R {\mathbf{v}_L^b} = 0 = {\mathbf{v}_L^a}' \mathbf{M}_R {\mathbf{v}_L^b}'$, and thus ${\mathbf{v}_L^a}' \mathbf{M}_R {\mathbf{v}_L^b}' \geq t_f$ as required.

So instead assume the opposite, in which case we define $w, \epsilon, \alpha_R$ as in \Cref{def:bilinear_capacity}.
Recall that $f(R, {\mathbf{v}_L^a}, {\mathbf{v}_L^b}), f(R, {\mathbf{v}_L^a}', {\mathbf{v}_L^b}')$ are computed as follows:

\begin{align}
{\mathbf{v}_L^a} \mathbf{M}_R {\mathbf{v}_L^b} &= \sum_{i=1}^{\delta_L} \sum_{j=1}^{\delta_L} {\mathbf{v}_L^a}[i] \mathbf{M}_R[i,j] {\mathbf{v}_L^b}[j] \\
{\mathbf{v}_L^a}' \mathbf{M}_R {\mathbf{v}_L^b}' &= \sum_{i=1}^{\delta_L} \sum_{j=1}^{\delta_L} {\mathbf{v}_L^a}'[i] \mathbf{M}_R[i,j] {\mathbf{v}_L^b}'[j]
\end{align}

We now consider 3 cases:

\begin{center}
\textbf{Case (1):} $\exists i, j \in \{ 1, ..., \delta_L \}$ such that ${\mathbf{v}_L^a}'[i] \geq \alpha$ and $\mathbf{M}_R[i,j] {\mathbf{v}_L^b}'[j] > 0$.
\end{center}

Then $\alpha \geq \alpha_R$, so ${\mathbf{v}_L^a}'[i] \geq \alpha_R$.
Also since $\mathbf{M}_R[i,j] > 0$, we have $\mathbf{M}_R[i,j] \geq w$, the least non-zero element of $\mathbf{M}_R$.
And since ${\mathbf{v}_L^b}'[j] > 0$, we have ${\mathbf{v}_L^b}'[j] \geq \epsilon$, the least non-zero element of $\bigcup_i \mathcal{X}_{L,i}$.
Now from the definition of $\alpha_R$, we have $\alpha_R \cdot w \cdot \epsilon \geq t_f$.
Subbing into this equation, we obtain ${\mathbf{v}_L^a}'[i] \mathbf{M}_R[i,j] {\mathbf{v}_L^b}'[j] \geq t_f$, which implies ${\mathbf{v}_L^a}' \mathbf{M}_R {\mathbf{v}_L^b}' \geq t_f$, as required.

\begin{center}
\textbf{Case (2):} $\exists i, j \in \{ 1, ..., \delta_L \}$ such that ${\mathbf{v}_L^b}'[j] \geq \alpha$ and ${\mathbf{v}_L^a}'[i] \mathbf{M}_R[i,j] > 0$.
\end{center}

This case is very similar to the first.
$\alpha \geq \alpha_R$, so ${\mathbf{v}_L^b}'[j] \geq \alpha_R$.
Also since $\mathbf{M}_R[i,j] > 0$, we have $\mathbf{M}_R[i,j] \geq w$, the least non-zero element of $\mathbf{M}_R$.
And since ${\mathbf{v}_L^a}'[i] > 0$, we have ${\mathbf{v}_L^a}'[i] \geq \epsilon$, the least non-zero element of $\bigcup_i \mathcal{X}_{L,i}$.
Now from the definition of $\alpha_R$, we have $\alpha_R \cdot w \cdot \epsilon \geq t_f$.
Subbing into this equation, we obtain ${\mathbf{v}_L^b}'[j] \mathbf{M}_R[i,j] {\mathbf{v}_L^a}'[i] \geq t_f$, which implies ${\mathbf{v}_L^a}' \mathbf{M}_R {\mathbf{v}_L^b}' \geq t_f$, as required.

\begin{center}
\textbf{Case (3):} $\forall i, j \in \{ 1, ..., \delta_L \}$, we have (${\mathbf{v}_L^a}[i] = {\mathbf{v}_L^a}'[i]$ or $\mathbf{M}_R[i,j] {\mathbf{v}_L^b}'[j] = 0$) and (${\mathbf{v}_L^b}[j] = {\mathbf{v}_L^b}'[j]$ or ${\mathbf{v}_L^a}'[i] \mathbf{M}_R[i,j] = 0$).
\end{center}
To see that the above covers all remaining options, recall that, from \Cref{prop:max_sum_lower_bound}, for all $v \in \texttt{enc}(D), \forall i \in \{ 1, ..., \delta_L \}$, either $\mathbf{v}_L[i] = \mathbf{v}_L'[i]$ or $\mathbf{v}_L'[i] \geq \alpha$.

The conditions in the case imply that $\forall i, j \in \{ 1, ..., \delta_L \}$, either (${\mathbf{v}_L^a}[i] = {\mathbf{v}_L^a}'[i]$ and ${\mathbf{v}_L^b}[j] = {\mathbf{v}_L^b}'[j]$) or ${\mathbf{v}_L^a}'[i] \mathbf{M}_R[i,j] {\mathbf{v}_L^b}'[j] = 0$.
Now let $i,j \in \{ 1, ..., \delta_L \}$. We consider a variety of sub-cases, proving that, in any of them, the following equality holds:

\begin{align} \label{align:bilinear_capacity_thm:equality}
{\mathbf{v}_L^a}[i] \mathbf{M}_R[i,j] {\mathbf{v}_L^b}[j] = {\mathbf{v}_L^a}'[i] \mathbf{M}_R[i,j] {\mathbf{v}_L^b}'[j] .
\end{align}

If ${\mathbf{v}_L^a}[i] = {\mathbf{v}_L^a}'[i]$ and ${\mathbf{v}_L^b}[j] = {\mathbf{v}_L^b}'[j]$, \Cref{align:bilinear_capacity_thm:equality} holds.
So instead consider ${\mathbf{v}_L^a}'[i] \mathbf{M}_R[i,j] {\mathbf{v}_L^b}'[j] = 0$, which is necessitated in the alternative.
At least one of these terms must be zero.

If $\mathbf{M}_R[i,j] = 0$, \Cref{align:bilinear_capacity_thm:equality} trivially holds.
Likewise, if ${\mathbf{v}_L^a}[i] = 0$ or ${\mathbf{v}_L^b}[j] = 0$, \Cref{align:bilinear_capacity_thm:equality} trivially holds, so assume to the contrary that both are non-zero.
Now if ${\mathbf{v}_L^a}'[i] = 0$, then ${\mathbf{v}_L^a}'[i] \not= {\mathbf{v}_L^a}[i]$, so we conclude from \Cref{align:bilinear_capacity_thm:equality} that ${\mathbf{v}_L^a}'[i] \geq \alpha$.
This is a contradiction, since $\alpha > 0$.
We reach a similar contradiction if ${\mathbf{v}_L^a}'[i] = 0$.

Thus, $\forall i,j \in \{ 1, ..., \delta_L \}$, \Cref{align:bilinear_capacity_thm:equality} holds.
From this it trivially follows that ${\mathbf{v}_L^a} \mathbf{M}_R {\mathbf{v}_L^b} = {\mathbf{v}_L^a}' \mathbf{M}_R {\mathbf{v}_L^b}'$.
Then since ${\mathbf{v}_L^a} \mathbf{M}_R {\mathbf{v}_L^b} \geq t_f$, so also ${\mathbf{v}_L^a}' \mathbf{M}_R {\mathbf{v}_L^b}' \geq t_f$, as required.

\end{proof}

\subsection{Theorem \ref{thm:max_sum_rule_shape}} \label{app:proof:thm:max_sum_rule_shape}
We can finally prove the main result for monotonic max-sum GNNs and non-negative bilinear scoring functions, which relies on \Cref{thm:capacity_maintains_equality} and is similar to \Cref{thm:max_rule_shape}.
The theorem shows that an equivalent program can be constructed for any monotonic max-sum GNN $\mathcal{N}$ and non-negative bilinear scoring function $f$ using rules taken from the finite space of all $(L, |\text{Col}| \cdot \delta_\mathcal{N} \cdot \mathcal{C}_{\mathcal{N}, f})$-tree-like rules.

\begin{theorem}
Let $\mathcal{N}$ be a monotonic max-sum GNN with unbounded activation functions and $f$ a non-negative bilinear scoring function.
Let $\mathcal{P}_\mathcal{N}$ be the Datalog program containing, up to variable renaming, each $(L, |\text{Col}| \cdot \delta_\mathcal{N} \cdot \mathcal{C}_{\mathcal{N}, f})$-tree-like rule that is sound for $(\mathcal{N}, f)$, where $\delta_\mathcal{N} = \text{max}(\delta_0, ..., \delta_L)$. Then $T_{\mathcal{N}, f}$ and $\mathcal{P}_\mathcal{N}$ are equivalent.
\end{theorem}

\begin{proof}
We show that $T_{\mathcal{N}, f}(D) = T_{P_\mathcal{N}}(D)$ holds for every dataset $D$.
By definition of $P_\mathcal{N}$, each rule in $P_\mathcal{N}$ is sound for $(\mathcal{N}, f)$, so we have $T_r(D) \subseteq T_{\mathcal{N}, f}(D)$ for every $r \in P_\mathcal{N}$, and thus $T_{P_\mathcal{N}}(D) \subseteq T_{\mathcal{N}, f}(D)$.

So we let $D$ be a dataset and prove $T_{\mathcal{N}, f}(D) \subseteq T_{P_\mathcal{N}}(D)$.
First, let $\mathcal{N}'$ be the GNN obtained from $\mathcal{N}$ by replacing $k_\ell$ with the capacity $\mathcal{C}_\ell$ for each $\ell \in \{1, ..., L\}$.
Then from \Cref{thm:capacity_maintains_equality}, it holds that $T_{\mathcal{N}, f}(D) = T_{\mathcal{N}', f}(D)$, so it suffices to show that $T_{\mathcal{N'}, f}(D) \subseteq T_{P_\mathcal{N}}(D)$.

Let $\alpha$ be an arbitrary binary fact in $T_{\mathcal{N'}, f}(D)$.
Observe that $\alpha$ can be of the form $R(s, t)$, for distinct constants $s, t$, or of the form $R(s, s)$, for a constant $s$.
We will denote these two cases as \textbf{(Case 1)} and \textbf{(Case 2)}, respectively. We will prove both cases simultaneously, making a separate argument when each case needs to be treated differently.

We construct a $(L, |\text{Col}| \cdot \delta_\mathcal{N} \cdot \mathcal{C}_{\mathcal{N}, f})$-tree-like rule $r$ (potentially with inequalities), such that $\alpha \in T_r(D)$ and $r$ is sound for $(\mathcal{N}', f)$.
Since from \Cref{thm:capacity_maintains_equality}, $T_{\mathcal{N}, f}(D_1) = T_{\mathcal{N}', f}(D_1)$ for every dataset $D_1$, $r$ being sound for $(\mathcal{N}', f)$ implies that $r$ is sound for $(\mathcal{N}, f)$.
So, together, these will imply that $\alpha \in T_{P_\mathcal{N}}(D)$, as required for the proof.

Let $G = (V, E, \lambda)$ be the canonical encoding of $D$ and let $\lambda_0, ..., \lambda_L$ the vertex labeling functions arising from the application of $\mathcal{N}'$ to $G$.

\begin{center}
\textbf{Setup of Construction}
\end{center}

We now construct a binary atom $H$, a conjunction $\Gamma$, a substitution $\nu$ from the variables of $\Gamma$ to the set of constants in $D$, and a graph $U$ (without vertex labels) where each vertex in $U$ is of the form $u^x$ for $x$ a variable and each edge has a colour in Col.
We also construct mappings $M_{c,\ell,j}: U \to 2^V$ for each $c \in \text{Col}$, $1 \leq \ell \leq L$, and $1 \leq j \leq \delta_{\ell - 1}$.
Intuitively, $M_{c,\ell,j}(u^x)$ represents the $c$-coloured neighbours of $v^t$ that contribute to the result of the max-sum aggregation in vector index $j$ at layer $\ell$ (where $v^t$ is the vertex corresponding to the constant $\nu(x)$).
We initialise $\Gamma$ as the empty conjunction, and we initialise $\nu$ and each $M_{c,\ell,j}$ as empty mappings.

We will assign to each vertex in $U$ a \emph{level} between $0$ and $L$. We will identify a single vertex from $U$ as the \emph{first root} vertex. Furthermore, if $\alpha$ is of the form $R(s, t)$, then we identify another distinct vertex from $U$ as the \emph{second root} vertex.

In the remainder of the proof, we use letters $s, t$ for constants in $D$, letters $x, y$ for variables, letters $v, w$ for the vertices in $V$, and (possibly indexed) letter $u$ for the vertices in $U$.
Our construction is by induction from level $L$ down to level $1$. The base case defines one or two vertices of level $L$. Then, for each
$1 \leq \ell \leq L$, the induction step considers the vertices of level $\ell$ and defines new vertices of level $\ell - 1$.

\begin{center}
\textbf{Construction: Base Case}
\end{center}

\textbf{(Case 1)} If $\alpha$ is of the form $R(s, t)$ for a binary predicate $R$, then $V$ contains vertices $v^s, v^t$.
We introduce fresh variables $x, y$, and define $\nu(x) = s, \nu(y) = t$;
we define $H = R(x, y)$;
we introduce vertices $u^x, u^y$ of level $L$;
and we make $u^x$ the first root vertex and $u^y$ the second root vertex.
Finally, we extend $\Gamma$ with atom $U(x)$ for each $U(s) \in D$, and also extend $\Gamma$ with atom $U(y)$ for each $U(t) \in D$.

\textbf{(Case 2)} If on the other hand, $\alpha$ is of the form $R(s, s)$, then $V$ contains a vertex $v^s$.
We introduce fresh variables $x, y$, and define $\nu(x) = s, \nu(y) = s$;
we define $H = R(x, y)$;
we introduce vertices $u^x, u^y$ of level $L$;
and we make $u^x$ the first root vertex and $u^y$ the second root vertex.
Finally, we extend $\Gamma$ with atoms $U(x)$ and $U(y)$ for each $U(s) \in D$.

\begin{center}
\textbf{Construction: Induction Step}
\end{center}

For the induction step, consider $1 \leq \ell \leq L$ and assume that all vertices of level greater than or equal to $\ell$ have been
already defined.
We then consider each vertex of the form $u^x$ of level $\ell$. Let $t = \nu(x)$.
For each colour $c \in \text{Col}$, each layer $1 \leq \ell' < \ell$, and each dimension $j \in \{ 1, ..., \delta_{\ell' - 1} \}$, let

\begin{align} \label{align:m_set_def}
&M_{c, \ell', j}(u^x) = \{ w ~|~ (v^t, w) \in E^c \text{ and } \mathbf{w}_{\lambda_{\ell'}}[j] \\
&\text{ contributes to max-$\mathcal{C}_{\ell'}$-sum}(\multisetopen \mathbf{w}_{\lambda_{\ell'}}[j] ~|~ (v^t, w) \in E^c \multisetclose) \} .
\end{align}

Intuitively, each $w$ picks out a $c$-coloured neighbour of $v^t$ that is among the $\mathcal{C}_{\ell'}$ largest in feature vector index $j$ at layer $\ell'$.
Note that at least one such set exists; however, it may not be unique, in which case any set satisfying \Cref{align:m_set_def} may be chosen.
Also, each vertex of $M_{c, \ell', j}(u^x)$ is of the form $v^{s_n}$ for some constant $s_n$ in $D$, where $s_n \not= s_m$ for all $1 \leq n < m \leq |M_{c, \ell', j}(u^x)|$.

We then introduce a fresh variable $y_n$ and define $\nu(y_n) = s_n$.
We introduce a vertex $u^{y_n}$ of level $\ell - 1$ and an edge $E^c(u^x, u^{y_n})$ to $U$.
Finally, we append to $\Gamma$ the conjunction

\begin{align}
\bigwedge_{n=1}^{|W|} \biggl(E^c(x, y_n) ~\land ~\bigwedge_{U(s_n) \in D} U(y_n) \biggr) \\
~\land~ \bigwedge_{1 \leq n < m \leq |W|} y_n \not\approx y_m  ,
\end{align}

where $W = M_{c, \ell', j}(u^x)$ and $U$ stands in for arbitrary unary predicates.
Since $M_{c, \ell', j}(u^x)$ contains at most $\mathcal{C}_{\ell'}$ vertices, this step adds at most $|\text{Col}| \cdot \delta_{\ell' - 1} \cdot \mathcal{C}_{\ell'} \cdot \ell$ new successors of $u^x$.
This completes the inductive construction.

\begin{center}
\textbf{Rule Shape and Immediate Consequences}
\end{center}

At this point in the proof, $\Gamma$ consists of a conjunction of two formulas, one of which is $(L, |\text{Col}| \cdot \delta_\mathcal{N} \cdot \mathcal{C}_{\mathcal{N}, f})$-tree-like for $x$ and the other of which is $(L, |\text{Col}| \cdot \delta_\mathcal{N} \cdot \mathcal{C}_{\mathcal{N}, f})$-tree-like for $y$ (where $x, y$ were defined in the base case).
We also have $H = R(x, y)$, so $\Gamma \rightarrow H$ is a $(L, |\text{Col}| \cdot \delta_\mathcal{N} \cdot \mathcal{C}_{\mathcal{N}, f})$-tree-like rule.
We define this rule to be $r$.

Furthermore, the construction of $\nu$ means $D \models \Gamma \nu$.
So we have $H \nu \in T_r(D)$, with $H \nu = \alpha$, so $\alpha \in T_r(D)$, as required.

Note that if $s, t$ are such that they are not mentioned in any fact in $D$, then $\Gamma$ is empty, and $r$ will then have the form $\top \rightarrow R(x, y)$.
Under our definitions, $r$ is still a $(L, |\text{Col}| \cdot \delta_\mathcal{N} \cdot \mathcal{C}_{\mathcal{N}, f})$-tree-like rule in this case, since $\top$ is tree-like for $x$ and $y$.
This applies similarly if only one of $s, t$ is not mentioned in any fact in $D$.

\begin{center}
\textbf{Rule Soundness}
\end{center}

To complete the proof, we show that $r$ is sound for $(\mathcal{N}', f)$.
For this purpose, let $D'$ be an arbitrary dataset and $\alpha'$ an arbitrary ground atom such that $\alpha' \in T_r(D')$.
Then there exists a substitution $\nu'$ such that $D' \models \Gamma \nu'$.

Let the graph $G = (V', E', \lambda')$ be the canonical encoding of $D'$ and $\lambda_0', ..., \lambda_L'$ the functions labelling the vertices of $G'$ when $\mathcal{N}'$ is applied to it.
We will use letters $p, q, q'$ to denote vertices of $V'$.

We prove the following statement by induction, which relies on the monotonicity of $\mathcal{N}'$: for each $0 \leq \ell \leq L$ and each vertex $u^x$ of $U$ whose level is at least $\ell$, we have $\mathbf{v}_{\lambda_\ell}[i] \leq \mathbf{p}_{\lambda_\ell'}[i]$ for each $i \in \{ 1, ..., \delta_\ell \}$, where $v = v^{\nu(x)}$ and $p = v^{\nu'(x)}$.

\begin{center}
\textbf{Monotonicity: Base Case}
\end{center}

For the base case, $\ell = 0$, consider an arbitrary $1 \leq i \leq \delta_0$ and $u^x \in U$.
Let $v = v^{\nu(x)}$ and $p = v^{\nu'(x)}$.
Since $\ell = 0$, values $\mathbf{v}_{\lambda_0}[i], \mathbf{p}_{\lambda_0'}[i]$ come from the canonical encoding, so $\mathbf{v}_{\lambda_0}[i] \in \{ 0, 1 \}$ and $\mathbf{p}_{\lambda_0'}[i] \in \{ 0, 1 \}$.

Thus, we prove that $\mathbf{v}_{\lambda_0}[i] = 1$ implies $\mathbf{p}_{\lambda_0'}[i] = 1$.
By definition of the canonical encoding, $\mathbf{v}_{\lambda_0}[i] = 1$ implies $U_i(x \nu) \in D$.
From the construction of $\Gamma$, we have $U_i(x) \in \Gamma$.
Also, by definition of $\nu'$, we have $D' \models \Gamma \nu'$.
So $U_i(x \nu') \in D'$ and $\mathbf{p}_{\lambda_0'}[i] = 1$, as required.

\begin{center}
\textbf{Monotonicity: Induction Step}
\end{center}

For the induction step, assume that the property holds for some $\ell - 1$.
Consider an arbitrary vertex $u^x \in U$ with level at least $\ell$, $c \in \text{Col}$, and $i \in \{ 1, ..., \delta_i \}$.
Let $v = v^{\nu(x)}$ and $p = v^{\nu'(x)}$.
Then $\mathbf{v}_{\lambda_\ell}[i], \mathbf{p}_{\lambda_\ell'}[i]$ are computed as follows:

\begin{align} \label{eq:v_max_sum_update}
&\mathbf{v}_{\lambda_\ell}[i] = \sigma_\ell \biggl(
\mathbf{b}_\ell[i] +
\sum_{j=1}^{\delta_{\ell-1}}
\mathbf{A}_\ell[i, j] \mathbf{v}_{\lambda_{\ell - 1}}[j] +
\sum_{c \in \text{Col}} \\
&\sum_{j=1}^{\delta_{\ell - 1}}
\mathbf{B}_\ell[i, j]
\text{max-$\mathcal{C}_\ell$-sum} \multisetopen
\textbf{w}_{\lambda_{\ell - 1}}[j] ~|~ (v, w) \in E^c
\multisetclose
\biggr)
\end{align}

\begin{align} \label{eq:p_max_sum_update}
&\mathbf{p}_{\lambda_\ell'}[i] = \sigma_\ell \biggl(
\mathbf{b}_\ell[i] +
\sum_{j=1}^{\delta_{\ell-1}}
\mathbf{A}_\ell[i, j] \mathbf{p}_{\lambda_{\ell - 1}'}[j] +
\sum_{c \in \text{Col}} \\
&\sum_{j=1}^{\delta_{\ell - 1}}
\mathbf{B}_\ell[i, j]
\text{max-$\mathcal{C}_\ell$-sum} \multisetopen
\textbf{q}'_{\lambda'_{\ell - 1}}[j] ~|~ (p, q') \in E'^c
\multisetclose
\biggr)
\end{align}

The induction assumption ensures $\mathbf{v}_{\lambda_{\ell - 1}}[j] \leq \mathbf{p}_{\lambda_{\ell - 1}'}[j]$ for each $1 \leq j \leq \delta_{\ell-1}$.
Also for each $c \in \text{Col}$ and $1 \leq j \leq \delta_{\ell - 1}$, we have $\text{max-$\mathcal{C}_\ell$-sum} \multisetopen \textbf{w}_{\lambda_{\ell - 1}}[j] ~|~ (v, w) \in E^c \multisetclose = \sum_{w \in W} \mathbf{w}_{\lambda_{\ell - 1}}[j]$, where $W = M_{c, \ell - 1, j}(u^x)$.

Recall that $W$ has the form $\{ v^{s_1}, ..., v^{s_{|W|}} \}$, where ${s_1}, ..., {s_{|W|}}$ are constants in $D$
Furthermore, by construction of $U$, there are $|W|$ distinct vertices $u^{y_1}, ..., u^{y_{|W|}}$ in $U$ of level $\ell - 1$ such that $\nu(y_n) = s_n$ and $E^c(u^x, u^{y_n})$ is in $U$ for each $1 \leq n \leq |W|$.

Furthermore, $\Gamma$ contains the atoms $E^c(x, y_1), ..., E^c(x, y_{|W|})$ and inequalities $y_n \not\approx y_m$ for $1 \leq n < m \leq |W|$.
We then have $E^c(\nu'(x), \nu'(y_n)) \in D'$ and $\nu'(y_n) \not= \nu'(y_m)$ for $1 \leq n < m \leq |W|$.
Thus, $W' = \{ v^{\nu'(y_1)}, ..., v^{\nu'(y_{|W|})} \}$ is a set of $|W|$ distinct $c$-neighbours of $v^{\nu'(x)}$ in $G'$.

The induction assumption means that for $w = v^{\nu(y_n)}$ and $q = v^{\nu'(y_n)}$, we have $\mathbf{w}_{\lambda_{\ell - 1}}[j] \leq \mathbf{q}_{\lambda'_{\ell - 1}}[j]$.
So $\sum_{w \in W} \mathbf{w}_{\lambda_{\ell - 1}}[j] \leq \sum_{q \in W'} \mathbf{q}_{\lambda_{\ell - 1}'}[j]$.

Thus, by the computations in \Cref{eq:v_max_sum_update} and \Cref{eq:p_max_sum_update}, and since all elements of $\mathbf{A}_\ell$ and each $\mathbf{B}_\ell^c$ are non-negative and $\sigma_\ell$ is monotonically increasing, we have $\mathbf{v}_{\lambda_\ell}[i] \leq \mathbf{p}_{\lambda_\ell'}[i]$, as required.

\begin{center}
\textbf{Conclusion}
\end{center}

Recall that $\alpha' \in T_r(D')$, so $\alpha'$ has the form $R(s', t')$ with $s' = \nu'(x)$ and $t' = \nu'(y)$.

\textbf{(Case 1)} if $\alpha$ has the form $R(s, t)$, then $R(s, t) \in T_{\mathcal{N}', f}(D)$ for $s = \nu(x)$ and $t = \nu(y)$, since $\alpha = R(s, t)$ is the original binary fact we assumed that the model produces from $D$.

Now let $v = v^s, w = v^t, p = v^{s'}, q = v^{t'}$.
Then $R(s, t) \in T_{\mathcal{N}', f}(D)$ implies $f(R, \mathbf{v}_{\lambda_L}, \mathbf{w}_{\lambda_L}) \geq t_f$.

The above property ensures that $\mathbf{v}_{\lambda_L}[i] \leq \mathbf{p}_{\lambda_L'}[i]$ and $\mathbf{w}_{\lambda_L}[i] \leq \mathbf{q}_{\lambda_L'}[i]$ for all $i \in \{ 1, ..., \delta_L \}$.
Thus, since $f$ is monotonically increasing, we have $f(R, \mathbf{v}_{\lambda_L}, \mathbf{w}_{\lambda_L}) \leq f(R, \mathbf{p}_{\lambda_L'}, \mathbf{q}_{\lambda_L'})$.

So $f(R, \mathbf{p}_{\lambda_L'}, \mathbf{q}_{\lambda_L'}) \geq t_f$, which implies that $R(s', t') \in T_{\mathcal{N}', f}(D')$, as required to show the soundness of the rule.

\textbf{(Case 2)} if on the other hand, $\alpha$ has the form $R(s, s)$, then $R(s, s) \in T_{\mathcal{N}', f}(D)$ for $s = \nu(x)$, since $\alpha = R(s, s)$ is the original binary fact we assumed that the model produces from $D$.

Now let $v = v^s, p = v^{s'}, q = {v^t}'$.
Then $R(s, s) \in T_{\mathcal{N}', f}(D)$ implies that $f(R, \mathbf{v}_{\lambda_L}, \mathbf{v}_{\lambda_L}) \geq t_f$.
Recall that $\nu(x) = \nu(y) = s$, so we have $v = v^{\nu(x)} = v^{\nu(y)}$, $p = v^{\nu'(x)}$, and $q = v^{\nu'(y)}$.

Thus, the above property ensures that $\mathbf{v}_{\lambda_L}[i] \leq \mathbf{p}_{\lambda_L'}[i]$ and $\mathbf{v}_{\lambda_L}[i] \leq \mathbf{q}_{\lambda_L'}[i]$ for all $i \in \{ 1, ..., \delta_L \}$.
Thus, since $f$ is monotonically increasing, we have $f(R, \mathbf{v}_{\lambda_L}, \mathbf{v}_{\lambda_L}) \leq f(R, \mathbf{p}_{\lambda_L'}, \mathbf{q}_{\lambda_L'})$.

So $f(R, \mathbf{p}_{\lambda_L'}, \mathbf{q}_{\lambda_L'}) \geq t_f$, which implies that $R(s', t') \in T_{\mathcal{N}', f}(D')$, as required to show the soundness of the rule.
\end{proof}
\section{Full Experiment Results} \label{app:full_results}
Experimental findings for monotonic sum GNNs are given in \Cref{tab:results:sum_gnns}.
Randomly sampled sound rules, with one or two body atoms, from monotonic max GNNs using a RESCAL decoder are shown in \Cref{tab:results:sample_rules}.
\changemarker{The longest time taken for searching and checking the rule space for sound rules on a single run was on fb237v1 using NAM, with 3h9m23s.
The shortest was on nellv1 using DistMult, with 15s.
The mean time taken to perform rule extraction and checking was 155s.
Process memory usage peaked around 370MB when performing rule extraction.}

%
%
%
%
%
%
%
%
%
%
%
%
%
%
% Sum GNNs
\begin{table*}
\centering
\resizebox{2.11\columnwidth}{!}{
\begin{tabular}{lll|rrrrrr|r|rr}
\toprule
Dataset & Decoder & Model & \%Acc & \%Prec & \%Rec & \%F1 & AUPRC & Loss & \%SO & \#1B & \#2B \\

\midrule
WN-hier
& DistMult
& Standard
& 90.04 & 89.37 & 90.9  & 90.12 & 0.9284 & 0.11 & -     & -     & -    \\
&& Monotonic
& 84.81 & 77.85 & 97.3  & 86.5  & 0.8464 & 1.48 & 76    & 186   & 28644 \\
& RESCAL
& Standard
& 92.51 & 90.34 & 95.28 & 92.71 & 0.9493 & 0.09 & -     & -     & -    \\
&& Monotonic
& 86.18 & 79.4  & 97.72 & 87.61 & 0.827  & 1.34 & 72    & 150   & 24000 \\
& NAM
& Standard
& 91.64 & 89.73 & 94.06 & 91.83 & 0.9325 & 0.08 & -     & -     & -    \\
&& Monotonic
& 78.4  & 72.66 & 93.68 & 81.54 & 0.7961 & 1.42 & 92    & 306   & 43699 \\
& TuckER
& Standard
& 92.11 & 90.74 & 93.84 & 92.25 & 0.9462 & 0.09 & -     & -     & -    \\
&& Monotonic
& 79.84 & 71.66 & 99.42 & 83.21 & 0.749  & 1.50 & 92    & 235   & 36357 \\

\midrule
WN-sym
& DistMult
& Standard
& 96.7  & 96.08 & 97.38 & 96.72 & 0.9944 & 0.08 & -     & -     & -    \\
&& Monotonic
& 91.62 & 93.39 & 89.58 & 91.45 & 0.9596 & 1.31 & 12    & 21    & 3458 \\
& RESCAL
& Standard
& 96.09 & 95.52 & 96.72 & 96.11 & 0.9928 & 0.06 & -     & -     & -    \\
&& Monotonic
& 93.11 & 88.73 & 98.76 & 93.48 & 0.9483 & 1.18 & 80    & 72    & 11287 \\
& NAM
& Standard
& 94.98 & 93.65 & 96.52 & 95.05 & 0.9877 & 0.06 & -     & -     & -    \\
&& Monotonic
& 84.44 & 78.83 & 94.4  & 85.85 & 0.8621 & 0.94 & 88    & 223   & 29578 \\
& TuckER
& Standard
& 95.92 & 95.68 & 96.2  & 95.93 & 0.9919 & 0.06 & -     & -     & -    \\
&& Monotonic
& 88.02 & 82.94 & 96    & 88.92 & 0.8934 & 1.31 & 80    & 96    & 15197 \\

\midrule
WN-cup\_nmhier
& DistMult
& Standard
& 75.06 & 73.84 & 78.09 & 75.74 & 0.8234 & 0.17 & -     & -     & -    \\
&& Monotonic
& 63.87 & 61.76 & 72.77 & 66.81 & 0.6252 & 1.88 & 0     & 182   & 27048 \\
& RESCAL
& Standard
& 78.93 & 76.55 & 83.42 & 79.83 & 0.8365 & 0.14 & -     & -     & -    \\
&& Monotonic
& 61.98 & 60.26 & 70.21 & 64.77 & 0.6171 & 1.82 & 20    & 187   & 28678 \\
& NAM
& Standard
& 78    & 74.09 & 86.33 & 79.67 & 0.823  & 0.14 & -     & -     & -    \\
&& Monotonic
& 60.02 & 57.78 & 85.96 & 68.39 & 0.5501 & 2.02 & 100   & 487   & 72270 \\
& TuckER
& Standard
& 79.07 & 75.29 & 86.6  & 80.53 & 0.84   & 0.14 & -     & -     & -    \\
&& Monotonic
& 59.2  & 60.05 & 59.31 & 58.41 & 0.5808 & 1.84 & 20    & 140   & 21367 \\

\midrule
fb237v1
& DistMult
& Standard
& 60.5  & 56.4  & 93.5  & 70.32 & 0.9579 & 0.02 & -     & -     & -    \\
&& Monotonic
& 79.65 & 76.3  & 86.1  & 80.89 & 0.6902 & 0.81 & -     & 41201 & -    \\
& RESCAL
& Standard
& 52.9  & 51.51 & 98.9  & 67.74 & 0.962  & 0.02 & -     & -     & -    \\
&& Monotonic
& 89.85 & 94.7  & 84.5  & 89.28 & 0.6106 & 0.72 & -     & 6608  & -    \\
& NAM
& Standard
& 58.45 & 55.18 & 92.9  & 69.15 & 0.9651 & 0.01 & -     & -     & -    \\
&& Monotonic
& 56.85 & 53.88 & 97.6  & 69.39 & 0.9652 & 0.32 & -     & 161875 & -    \\
& TuckER
& Standard
& 59.65 & 55.58 & 96.6  & 70.55 & 0.9534 & 0.01 & -     & -     & -    \\
&& Monotonic
& 89.5  & 90.64 & 88.2  & 89.38 & 0.629  & 0.86 & -     & 9233  & -    \\

\midrule
WN18RRv1
& DistMult
& Standard
& 72.91 & 65.47 & 99.52 & 78.83 & 0.8995 & 0.01 & -     & -     & -    \\
&& Monotonic
& 91.21 & 85.05 & 100   & 91.92 & 0.7914 & 0.79 & -     & 97    & 11929 \\
& RESCAL
& Standard
& 84.67 & 79.14 & 98.79 & 87.36 & 0.9175 & 0.01 & -     & -     & -    \\
&& Monotonic
& 94.97 & 92.06 & 98.42 & 95.14 & 0.7878 & 0.77 & -     & 50    & 6143 \\
& NAM
& Standard
& 64.67 & 59.68 & 100   & 74.43 & 0.898  & 0.01 & -     & -     & -    \\
&& Monotonic
& 77.58 & 69.12 & 100   & 81.72 & 0.9198 & 0.12 & -     & 201   & 21137 \\
& TuckER
& Standard
& 93.33 & 89.11 & 98.79 & 93.69 & 0.9285 & 0.01 & -     & -     & -    \\
&& Monotonic
& 94.73 & 92.03 & 97.94 & 94.89 & 0.7861 & 0.79 & -     & 51    & 6309 \\

\midrule
nellv1
& DistMult
& Standard
& 58.24 & 60.2  & 44.94 & 51.27 & 0.9392 & 0.09 & -     & -     & -    \\
&& Monotonic
& 60.35 & 55.77 & 100   & 71.61 & 0.8937 & 1.43 & -     & 794   & 153341 \\
& RESCAL
& Standard
& 65.65 & 67.7  & 63.76 & 65.02 & 0.9368 & 0.03 & -     & -     & -    \\
&& Monotonic
& 83.06 & 74.77 & 100   & 85.54 & 0.7208 & 1.15 & -     & 305   & 58784 \\
& NAM
& Standard
& 45.53 & 44.48 & 16.94 & 22.81 & 0.918  & 0.03 & -     & -     & -    \\
&& Monotonic
& 48.12 & 48.41 & 60.24 & 53.65 & 0.9292 & 1.65 & -     & 862   & 146403 \\
& TuckER
& Standard
& 61.88 & 66.46 & 51.53 & 57.17 & 0.9392 & 0.03 & -     & -     & -    \\
&& Monotonic
& 70    & 62.66 & 100   & 77    & 0.7511 & 1.30 & -     & 395   & 78152 \\

\bottomrule
\end{tabular}
} % end resize box
\caption{Results for sum GNNs. Loss is from the final epoch on the training set. AUPRC is from the validation set. Other metrics are computed on the test set. \%SO is the percentage of LogInfer rules that are sound for the model. \#1B and \#2B are the number of sound rules with one and two body atoms respectively.}
\label{tab:results:sum_gnns}
\end{table*}

%
%
%
%
%
%
%
%
%
%
%
%
% Rule Samples
\begin{table*}
\centering
% \resizebox{1\columnwidth}{!}{
\begin{tabular}{l|p{14.5cm}}
\toprule
Dataset & Sample Sound Rules\\

\midrule

WN-hier
& \_similar\_to(x,y) implies \_hypernym(x,y) \\
& \_has\_part(x,y) and \_member\_meronym(z,a) implies \_hypernym(x,a) \\
& \_verb\_group(x,y) implies \_derivationally\_related\_form(y,y) \\
& \_hypernym(y,x) and \_similar\_to(z,y) implies \_hypernym(x,z) \\
& \_has\_part(x,y) implies \_hypernym(y,x) \\
& \_member\_of\_domain\_usage(y,x) and \_similar\_to(y,z) implies \_member\_meronym(y,x) \\
& \_synset\_domain\_topic\_of(x,y) implies \_hypernym(x,y) \\
& \_instance\_hypernym(y,x) and \_synset\_domain\_topic\_of(z,y) implies \_has\_part(y,z) \\

\midrule

WN-sym
& \_synset\_domain\_topic\_of(x,y) implies \_synset\_domain\_topic\_of(y,x) \\
& \_hypernym(y,x) and \_similar\_to(z,y) implies \_hypernym(x,z) \\
& \_member\_of\_domain\_usage(x,y) implies \_hypernym(x,y) \\
& \_has\_part(x,y) and \_member\_meronym(z,a) implies \_hypernym(x,a) \\
& \_instance\_hypernym(x,y) implies \_instance\_hypernym(y,y) \\
& \_also\_see(y,x) and \_also\_see(z,y) implies \_also\_see(y,y) \\
& \_similar\_to(x,x) implies \_similar\_to(x,x) \\
& \_hypernym(x,y) and \_verb\_group(z,a) implies \_hypernym(a,y) \\

\midrule

WN-cup\_nmhier
& \_derivationally\_related\_form(x,y) implies \_derivationally\_related\_form(y,x) \\
& \_has\_part(x,y) and \_member\_meronym(z,a) implies \_hypernym(x,a) \\
& \_member\_of\_domain\_usage(x,y) implies \_member\_of\_domain\_usage(y,y) \\
& \_member\_of\_domain\_region(y,y) and \_member\_of\_domain\_region(z,y) implies \_hypernym(y,y) \\
& \_instance\_hypernym(x,y) implies \_has\_part(y,y) \\
& \_hypernym(x,y) and \_instance\_hypernym(y,z) implies \_instance\_hypernym(x,z) \\
& \_member\_meronym(x,y) implies \_hypernym(x,y) \\
& \_member\_of\_domain\_region(x,y) and \_hypernym(y,z) implies \_has\_part(y,y) \\

\midrule

fb237v1
& /sports/sports\_position/players./sports/sports\_team\_roster/team(x,y) implies /people/person/places\_lived./people/place\_lived/location(x,x) \\
& /education/educational\_degree/people\_with\_this\_degree./education/education/institution(x,y) implies /location/location/contains(y,x) \\
& /film/film/release\_date\_s./film/film\_regional\_release\_date/film\_release\_region(x,y) implies /award/award\_nominee/award\_nominations./award/award\_nomination/award\_nominee(x,y) \\
& /location/country/official\_language(x,y) implies /location/hud\_county\_place/county(y,y) \\
& /travel/travel\_destination/how\_to\_get\_here./travel/transportation/mode\_of\_transportation(x,y) implies /location/hud\_county\_place/county(y,x) \\

\midrule

WN18RRv1
& \_also\_see(x,y) implies \_derivationally\_related\_form(x,x) \\
& \_also\_see(y,x) and \_hypernym(y,y) implies \_also\_see(x,x) \\
& \_member\_meronym(x,y) implies \_derivationally\_related\_form(x,y) \\
& \_verb\_group(x,y) and \_has\_part(z,y) implies \_derivationally\_related\_form(y,x) \\
& \_synset\_domain\_topic\_of(x,y) implies \_derivationally\_related\_form(y,x) \\
& \_similar\_to(x,y) and \_has\_part(z,a) implies \_derivationally\_related\_form(y,x) \\
& \_derivationally\_related\_form(x,y) and \_verb\_group(y,y) implies \_hypernym(y,y) \\
& \_hypernym(x,y) implies \_derivationally\_related\_form(y,y) \\

\midrule

nellv1
& concept:organizationhiredperson(x,y) implies concept:organizationhiredperson(y,y) \\
& concept:agentbelongstoorganization(y,x) and concept:organizationterminatedperson(z,y) implies concept:worksfor(x,z) \\
& concept:subpartoforganization(x,x) implies concept:acquired(x,x) \\
& concept:televisionstationaffiliatedwith(y,x) and concept:organizationhiredperson(z,y) implies concept:organizationhiredperson(x,y) \\
& concept:organizationterminatedperson(x,y) implies concept:topmemberoforganization(y,y) \\
& concept:headquarteredin(x,y) and concept:organizationhiredperson(y,z) implies concept:headquarteredin(z,y) \\
& concept:subpartoforganization(x,y) and concept:subpartoforganization(z,y) implies concept:worksfor(x,z) \\

\bottomrule
\end{tabular}
% } % end resize box
\caption{Randomly sampled sound rules from monotonic max GNNs using a RESCAL decoder.}
\label{tab:results:sample_rules}
\end{table*}

\end{document}